\documentclass{article} 
\usepackage{iclr2027_conference,times}

\usepackage{amsmath,amsfonts,bm}

\def\eqref#1{equation~\ref{#1}}

\def\1{\bm{1}}

\DeclareMathAlphabet{\mathsfit}{\encodingdefault}{\sfdefault}{m}{sl}
\SetMathAlphabet{\mathsfit}{bold}{\encodingdefault}{\sfdefault}{bx}{n}

\usepackage{hyperref}
\usepackage{url}

\usepackage{amsmath}
\usepackage{cleveref}
\usepackage{bm}
\usepackage{amsmath}
\usepackage{todonotes}
\usepackage{wrapfig}
\usepackage{physics}
\usepackage{amssymb}
\usepackage{amsthm}
\usepackage[utf8]{inputenc} 
\usepackage[T1]{fontenc}    
\usepackage{booktabs}       
\usepackage{amsfonts}       
\usepackage{nicefrac}       
\usepackage{microtype}      
\usepackage{lipsum}
\usepackage{fancyhdr}       
\usepackage{graphicx}       
\usepackage{physics}
\usepackage[dvipsnames]{xcolor}
\usepackage[table,xcdraw]{xcolor}
\usepackage[most]{tcolorbox}

\hypersetup{colorlinks=true,citecolor=NavyBlue,linkcolor=NavyBlue,urlcolor=NavyBlue}

\newtcolorbox{promptbox}[1]{%
   breakable,
   enhanced,
   colback=gray!4,
   colframe=gray!55!black,
   boxrule=0.45pt,
   arc=2.5pt,
   left=6pt, right=6pt, top=5pt, bottom=5pt,
   title=\textsf{\small\bfseries #1},
   fonttitle=\sffamily\small\bfseries,
   coltitle=white,
   colbacktitle=gray!55!black,
   attach boxed title to top left={yshift=-2mm, xshift=4mm},
   boxed title style={colback=gray!55!black, sharp corners, boxrule=0pt},
   fontupper=\small\ttfamily,
   before upper={\setlength{\parindent}{0pt}\setlength{\parskip}{2pt}},
}

\newtcolorbox{examplebox}[1]{%
   breakable,
   enhanced,
   colback=blue!2,
   colframe=blue!35!black,
   boxrule=0.45pt,
   arc=2.5pt,
   left=6pt, right=6pt, top=5pt, bottom=5pt,
   title=\textsf{\small\bfseries #1},
   fonttitle=\sffamily\small\bfseries,
   coltitle=white,
   colbacktitle=blue!35!black,
   attach boxed title to top left={yshift=-2mm, xshift=4mm},
   boxed title style={colback=blue!35!black, sharp corners, boxrule=0pt},
   fontupper=\small,
   before upper={\setlength{\parindent}{0pt}\setlength{\parskip}{2pt}},
}

\newtheorem{theorem}{Theorem}[section]

\newtheorem{corollary}[theorem]{Corollary}

\newtheorem{assumption}[theorem]{Assumption}

\newtheorem{remark}[theorem]{Remark}

\title{
Harness Evolution as Learning: Approximation, Generalization, and Optimization Limits of Self-Improving Personal Agents}

\author{Zeyu Gan, Zixuan Gong, Yong Liu\thanks{Corresponding Author.} \\
Gaoling School of Artificial Intelligence\\
Renmin University of China\\
Beijing, China \\
\texttt{\{zygan,zxgong,liuyonggsai\}@ruc.edu.cn}
}

\iclrfinalcopy 
\begin{document}

\maketitle

\begin{abstract}
As the capabilities of large language models (LLMs) continue to advance, increasing attention is turning to how to translate their abilities into useful behavior. Personal agents bring this question into everyday settings, where models are expected to serve individual users and continually adapt to their preferences. With the underlying model held fixed, such adaptation relies on harness engineering: designing and evolving the surrounding layer that manages context, memory, tools, and execution. Despite rapid progress, the factors governing effective harness evolution remain insufficiently understood. To narrow this gap, we investigate three central questions concerning harness architecture, harness scale, and self-evolution algorithms through complementary empirical and theoretical analyses. Empirically, we introduce a preference-oriented benchmark and systematically characterize the capabilities and limitations of personal agents associated with these three dimensions. Theoretically, we formulate harness evolution as a learning problem and explain these phenomena through approximation, generalization, and optimization errors. Analyses of reachable policies, capacity under finite interaction evidence, and biased update dynamics provide theoretical accounts of the observed phenomena. Together, these results offer a unified perspective on the limits of personalization through harness evolution and inform future harness design. We open-source our code at \url{https://github.com/ZyGan1999/self-evolving-harness-as-learning}. 
\end{abstract}

\section{Introduction}
\label{sec:introduction}
As large language models (LLMs) become more capable, research attention increasingly extends from how to train a good model to how to use one to accomplish useful tasks~\citep{wang2024survey}.
LLM agents embody this direction by coupling language-model reasoning with tools, memory, and repeated interaction with an environment~\citep{yao2023reactsynergizingreasoningacting,NEURIPS2024_swe_agent}.
The resulting applications bring a concrete promise of everyday assistance: an agent might locate a receipt across applications, organize the files needed for a project, or diagnose and repair a failing software test.
Though extensive works turn parts of this promise into executable tasks, they also expose the difficulty of translating model capabilities into reliable actions~\citep{trivedi-etal-2024-appworld,ICLR2024_swe_bench}.
With the growing demand for commercial and open-source assistants, the design of more reliable and advanced agent systems has become an increasingly practical research question. 

A particularly relevant setting is the \emph{personal agent}: an assistant that works within a user's persistent workspace and is accessed through their own computers, servers, or everyday communication tools.
OpenClaw~\citep{openclaw} and Hermes Agent~\citep{hermes} connect assistance to messaging and recurring daily workflows, while Claude Code~\citep{claude_code} and Codex~\citep{codex} emphasize software development within the user's working environment. 
Personal agents must accommodate requirements that vary across users and persist across tasks, for example, a preferred message sign-off or conventions for payment notes, 
while following an API-backed deployment regime. 
In this regime, the agent runtime and personalization state are locally controlled, while the foundation model is remotely served and its parameters are not updated by the local agent.
Hence, adapting behavior through external state and procedures becomes a practical alternative to modifying the model itself~\citep{zhao2024expel}. One practical attempt is to make the model behave appropriately for an individual user over continued interaction by constructing an extra surrounding layer. 

The surrounding layer that constructs model inputs, manages persistent state, exposes tools, and controls execution is commonly called an \emph{agent harness}. Designing this layer is the subject of \emph{harness engineering}~\citep{li2026agentharness}.
Existing approaches span interaction interfaces, context retrieval and compaction, reusable skills, execution checks, and recovery procedures~\citep{NEURIPS2024_swe_agent,anthropic2025context,anthropic2025harnesses}.
For a personal agent, this layer also provides a place to encode user-specific requirements and revise them as new evidence arrives.
Such \emph{harness evolution} has become an active research direction, with particular interest in self-evolution: using an LLM to interpret interaction feedback and update the memory, skills, or code that govern subsequent behavior~\citep{ICLR2026_ACE,yang2026autoskillexperiencedrivenlifelonglearning,pan2026mstartaskdeservesmemory}.
The prospect is appealing: an assistant could become better adapted to its user simply through accumulated experience, without retraining its foundation model.

Yet the diversity of harness designs and evolution algorithms leaves three practical tensions, illustrated schematically in Figure~\ref{fig:intro}.
\textbf{First, the effectiveness of a harness depends on what the preference requires.}
Panel~(a) contrasts two everyday requests: a remembered instruction to use a friendly tone may suffice to shape an email draft, whereas honoring ``leave at least 30 minutes between my meetings'' requires checking calendar state and reasoning about time constraints before booking.
This contrast is consistent with evidence that memory designs transfer unevenly across tasks 
~\citep{pan2026mstartaskdeservesmemory,zhou2026gettingbetterworkingyou}.
\textbf{Second, increasing harness size does not necessarily improve performance.}
Panel~(b) illustrates how an agent can omit a stored preference despite maintaining an increasingly comprehensive memory document.
Accumulated instructions can interact with one another and interfere with the agent's output, so their benefits need not increase with their number~\citep{ICLR2025_do_llms_recognize,zhou2026teparevokingstalememories,anthropic2025context}.
\textbf{Third, repeated self-evolution does not guarantee sustained improvement.}
Panel~(c) illustrates one potential source of this difficulty: a correction about a missing signature is turned into an unsupported full-name requirement and an additional politeness rule, 
suggesting that additional iterations alone may not overcome the underlying bottlenecks~\citep{ICLR2026_ACE,nie2026tthetesttimeharnessevolution,liu2026adaptiveautoharnesssustainedselfimprovement}.
Although existing studies document individual failures and partial theoretical accounts have begun to emerge~\citep{gan2026blackboxsurveytheory}, a common basis for standardized, quantitative empirical study of all three tensions, alongside a convincing theoretical account of how they arise and relate to one another in personal-agent harness evolution, remains lacking.
This gap makes it difficult to determine when to change the harness mechanism, adjust its scale, or improve the evolution algorithm.

\begin{figure}[tp]
    \centering
    \includegraphics[width=\linewidth]{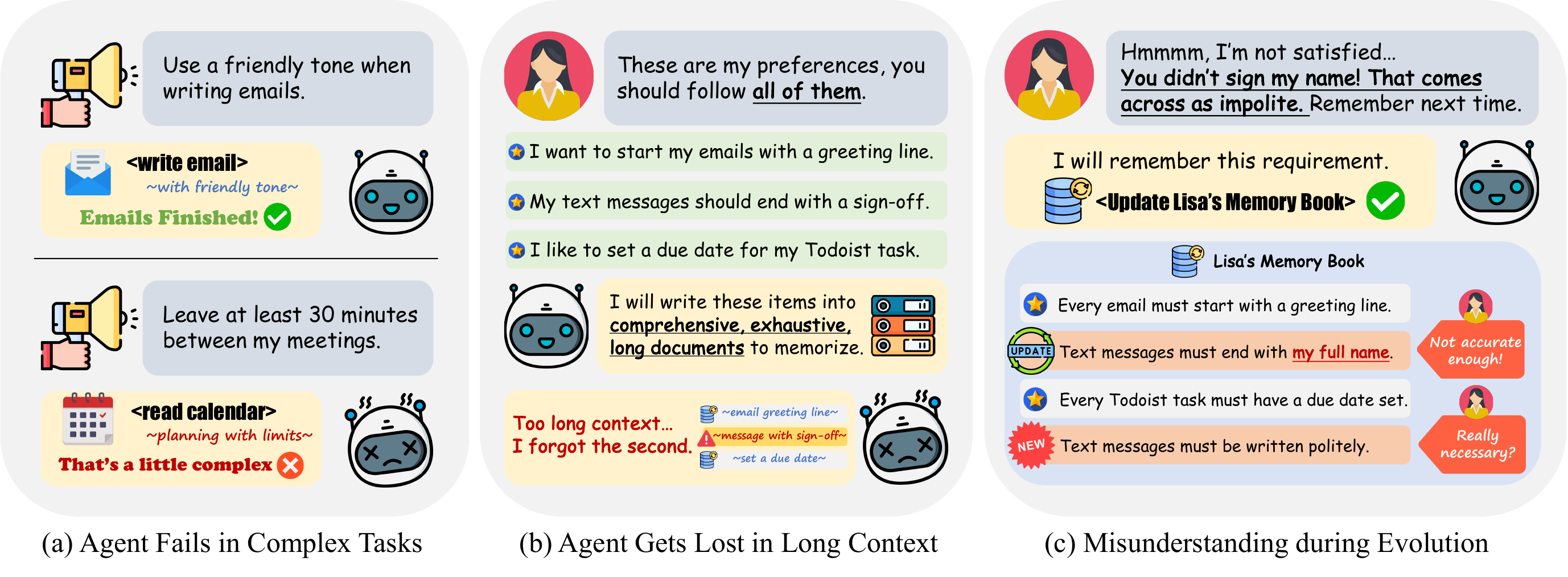}
    \caption{Schematic examples of three practical tensions in personal agent harness evolution: (a) effectiveness varies with preference requirements; (b) agents can overlook preferences stored in long memories; and (c) misinterpreted feedback can produce inaccurate or unnecessary memory updates.}

    \label{fig:intro}
    \vspace{-20pt}
\end{figure}

To address this gap, we study harness evolution from complementary empirical and theoretical perspectives.
Empirically, we introduce a preference-oriented benchmark that makes the three tensions measurable and enables their systematic characterization.
Theoretically, we cast harness evolution as a learning problem and organize its risk gap into \emph{approximation}, \emph{generalization}, and \emph{optimization} errors, following the classical learning-theoretic separation. 
Latent-concept, information-theoretic, and fixed-point analyses then provide conditional explanations for the empirical phenomena.

Our contributions are twofold.
\textbf{Empirically}, we provide a common testbed and controlled comparisons showing that explicit context reliably handles some local preferences while computational mechanisms improve more demanding ones, that memory expansion can yield non-monotonic compliance under the evaluated construction, and that several self-evolving recipes retain substantial oracle gaps.
\textbf{Theoretically}, 
we introduce a unified learning-theoretic framework for harness evolution, linking architecture, memory scale, and update dynamics to approximation, generalization, and optimization errors, with analyses of reachability limits, memory-capacity and estimation bounds, and residual error under biased updates. 
The combined perspective links observed failures to distinct questions about harness design and learning.
The remainder of the paper reviews related work in Section~\ref{sec:related-works}, presents the formulation in Section~\ref{sec:preliminary}, and develops the empirical and theoretical analyses in Section~\ref{sec:main-results}.
Finally, Section~\ref{sec:conclusion} draws a brief conclusion. 

\section{Related Works}
\label{sec:related-works}
\paragraph{Personal Agents. }
Personal agents increasingly operate in persistent user environments, from messaging-based assistants such as OpenClaw~\citep{openclaw} to workspace-based coding assistants such as Codex~\citep{codex} and Claude Code~\citep{claude_code}. 
Earlier research established memory as a substrate for persistent agent behavior~\citep{park2023generative,packer2024memgptllmsoperatingsystems}. 
Personalization also requires applying memory appropriately. 
\citet{ICLR2025_do_llms_recognize} evaluate whether models infer and follow user preferences across extended conversations. 
\citet{wang2026openclawrltrainagentsimply} take a complementary route by converting subsequent user and environment responses into training signals for online reinforcement learning. 
Our setting fixes the foundation model and studies how a local harness learns to express user preferences through the available interaction and feedback channels. 

\paragraph{Harness Engineering. }
An agent's behavior depends on how the harness system constructs context, exposes actions, maintains state, and checks outcomes. 
\citet{yao2023reactsynergizingreasoningacting} first interleave reasoning with environment actions, 
\citet{NEURIPS2024_swe_agent} further demonstrate the importance of interfaces for navigating, editing, and executing code. 
\citet{pmlr-v235-wang24h} then use executable code as a compositional action representation. 
Instead, \citet{xia2025agentless} structure software repair into localization, repair, and patch validation. 
Industrial accounts similarly emphasize context curation and compaction, persistent progress artifacts, and testing across sessions~\citep{anthropic2025context,anthropic2025harnesses}.
Verification and evaluation provide complementary evidence: 
\citet{ICLR2024_swe_bench} test repository-level repairs, while 
\citet{trivedi-etal-2024-appworld} evaluate tool-mediated tasks through environment outcomes. 

\paragraph{Self-Evolving Harness. }
The community has begun to apply automated adaptation to modify several components of a harness. 
\citet{NEURIPS2023_Reflexion} store verbal feedback in an episodic buffer, and 
\citet{zhao2024expel} extract reusable insights from task experience. 
\citet{agrawal2026gepareflectivepromptevolution} search over prompts using trajectory reflection and evaluated candidate selection. 
For persistent context, 
\citet{ICLR2026_ACE} use structured playbooks and a grow-and-refine process. 
\citet{zhou2026teparevokingstalememories} track the validity of keyed precedents and revoke stale memories. 
More recent approaches consider targeting dynamic user adaptation~\citep{yang2026autoskillexperiencedrivenlifelonglearning,zhou2026gettingbetterworkingyou}. 
Evolution can also change the mechanism that uses memory or controls actions~\citep{zhang2025memevolvemetaevolutionagentmemory,pan2026mstartaskdeservesmemory,lou2026autoharnessimprovingllmagents,nie2026tthetesttimeharnessevolution,gan2026statisticalpriorsimplicitpreferences}. 
Analytical treatments are also emerging. 
\citet{liu2026adaptiveautoharnesssustainedselfimprovement} decompose an oracle-harness gap into evolution and adaptation losses to guide harness construction. 

\section{Preliminary \& Problem Formulation}
\label{sec:preliminary}

We formulate personalization as learning a local harness around a frozen foundation model, and use a three-way risk decomposition to organize the questions studied in this paper. 


\subsection{Personal Agents and Local Harness}
\label{sec:personal-agents-and-local-harness}
Let $z=(x,e)$ denote a task, where $x$ is the natural-language instruction and $e$ contains the application state and available tools. A frozen foundation model $f$ induces a distribution over interaction traces $\tau=(a_1,o_1,\ldots,a_T,o_T)$, which we write as $\tau \sim \pi_f(\cdot \mid z). $
Here $a_t$ is an attempted action (for example, an API call) and $o_t$ is the observation returned by the environment. The notation emphasizes that $f$ is held fixed throughout the paper; all adaptation happens outside its parameters. 

\begin{wrapfigure}{r}{0.5\textwidth}
    \centering
    \includegraphics[width=0.5\textwidth]{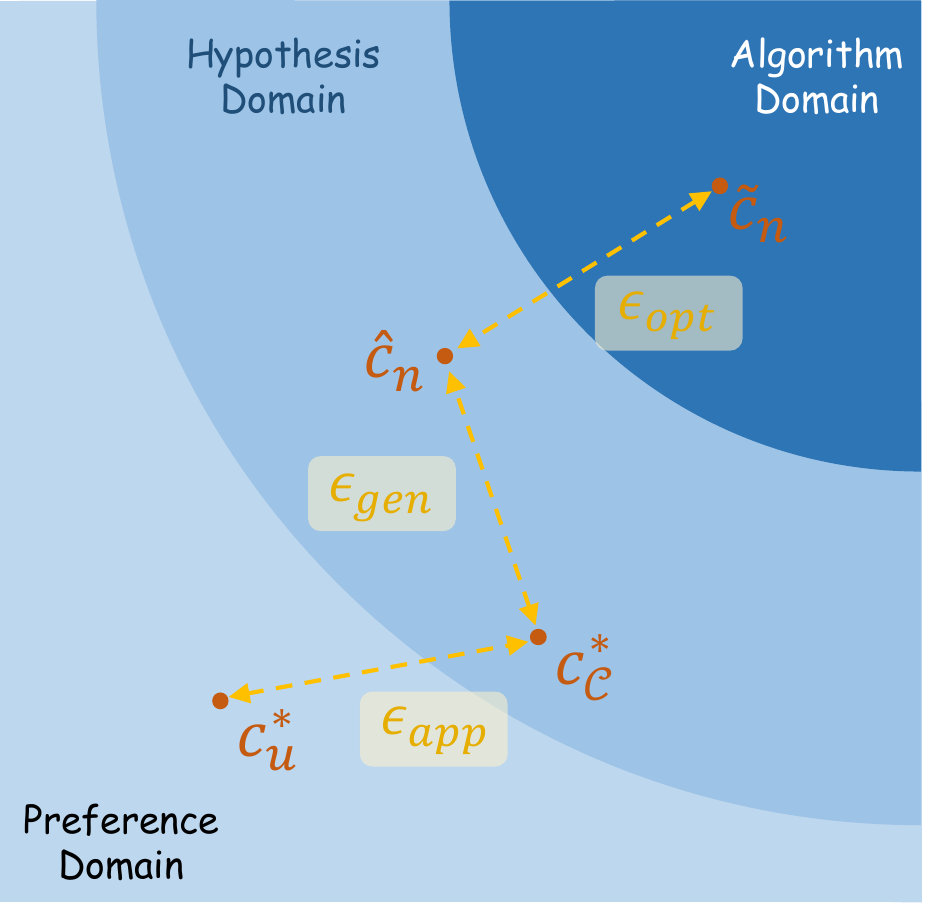}
    \caption{Harness evolution as learning. Approximation, generalization, and optimization errors separate the user optimum, class optimum, empirical optimum, and learned harness, respectively.}

    \label{fig:harness-as-learning}
    \vspace{-30pt}
\end{wrapfigure}

\emph{Harnesses} are local stateful transducers that mediate between the model and the environment. For a local state $s_t$ and model history $h_t$, a harness $c$ can construct the next model input, maintain auxiliary state, and optionally transform or reject an action: $(\bar h_t,\bar a_t,s_{t+1})=c(h_t,a_t,s_t,e). $
The composition $c\circ f$ therefore defines an induced policy $\pi_{c\circ f}$. 
This notation covers common personalization mechanisms: contextual memory and system-prompt edits change $\bar h_t$; checkers, routers, statistics, and autofill components can additionally change $\bar a_t$ or the available action set.

Specifically, we distinguish two broad hypothesis classes. A \emph{context harness} has $\bar a_t=a_t$ and acts only through an injected token sequence (memory, skills, or prompt text). A \emph{control harness} may compute over the interaction state or intervene in the action path, for example by returning a derived value, maintaining a counter, or rewriting a preference-owned field. 

\subsection{Harness Evolution as Learning}
\label{sec:harness-evolution-as-learning}
Let $f_u^*$ denote the ideal policy that would express the user's preferences if the model and the local harness were unrestricted. We evaluate a policy under the user's task distribution $\mathcal{P}_u$ with loss $\ell_u$ and risk $R(\pi)=\mathbb{E}_{z\sim\mathcal{P}_u,\,\tau\sim\pi(\cdot\mid z)}\left[\ell_u(\tau,z)\right]$. 
The personalization objective is consequently $\min_{c\in\mathcal{C}} R(\pi_{c\circ f})$, ideally approaching $R(\pi_{f_u^*})$.

The user-optimal policy is not observed directly. 
During interaction $i$, the agent produces a trace $\tau_i$ and receives weak user or environment feedback $y_i$. The learner observes the stream
$S_n=\{(z_i,\tau_i,y_i)\}_{i=1}^{n}$, where $y_i$ may be only acceptance/dissatisfaction, an artefact-level complaint, or a more informative correction. 

For a fixed harness class $\mathcal{C}$, define the population-optimal harness and the empirical-risk minimizer by
$c^*_{\mathcal{C}} \in \arg\min_{c\in\mathcal{C}} R(\pi_{c\circ f})$ and $\hat c_n \in \arg\min_{c\in\mathcal{C}} \widehat R_n(\pi_{c\circ f};S_n)$ respectively, where $\widehat R_n$ is the empirical risk induced by the observed feedback channel. An evolution algorithm is a model- and channel-dependent update operator
$c_i=U_f(c_{i-1};z_i,\tau_i,y_i)$, and $\tilde c_n=U_f(S_n)$,
with $\tilde c_n$ the harness actually deployed after $n$ interactions. 

\subsection{Error Decomposition: Approximation, Generalization and Optimization}
\label{sec:three-way-error-decomposition}
Consider the real risk gap between $\pi_{\tilde c_n\circ f}$ and $\pi_{f_u^*}$ during the harness evolution. By adding and subtracting the risks of $\pi_{c^*_{\mathcal{C} \circ f}}$ and $\pi_{\hat c_n \circ f}$, we obtain the following risk decomposition: 

\begin{equation}
\resizebox{0.9\linewidth}{!}{
$
\begin{aligned}
    R(\pi_{\tilde c_n\circ f})-R(\pi_{f_u^*})
    ={}\underbrace{R(\pi_{c^*_{\mathcal{C}}\circ f})-R(\pi_{f_u^*})}
        _{\displaystyle \epsilon_{\mathrm{app}}(\mathcal{C},f_u^*)}
    +\underbrace{R(\pi_{\hat c_n\circ f})-R(\pi_{c^*_{\mathcal{C}}\circ f})}
        _{\displaystyle \epsilon_{\mathrm{gen}}(\mathcal{C},n)}
    +\underbrace{R(\pi_{\tilde c_n\circ f})-R(\pi_{\hat c_n\circ f})}
        _{\displaystyle \epsilon_{\mathrm{opt}}(U_f;\mathcal{C},S_n)}.
    \nonumber
    \label{eq:error-decomposition}
\end{aligned}
$
}
\end{equation}

The first term ($\epsilon_{\mathrm{app}}$) is known as \textbf{Approximation Error}, and is a property of the reachable policy class. The second term ($\epsilon_{\mathrm{gen}}$) is known as \textbf{Generalization Error}, related to data and its effective capacity. And the third term ($\epsilon_{\mathrm{opt}}$) is often called \textbf{Optimization Error}, which is related to the evolution algorithm. 
As illustrated in Figure~\ref{fig:harness-as-learning}, the global optimal harness solution $c_u^*$ lies within the preference domain, while the optimal solution $c_{\mathcal{C}}^*$ in the hypothesis set $\mathcal{C}$ lies within the hypothesis domain. The gap between $c_u^*$ and $c_{\mathcal{C}}^*$ constitutes $\epsilon_{\mathrm{app}}$. The empirically optimal solution $\hat c_n$ also lies within the hypothesis domain, but is constrained by the limited interaction data, thus the gap between $\hat c_n$ and $c_{\mathcal{C}}^*$ constitutes $\epsilon_\mathrm{gen}$. Moreover, the realistic solution $\tilde c_n$ obtained during the evolution lies within the algorithm domain, the gap between $\tilde c_n$ and $\hat c_n$ constitutes $\epsilon_\mathrm{opt}$. 

These three types of errors characterize gaps from different perspectives within the harness evolution process. Approximation error suggests that the specific implementation of the harness inherently determines the upper limit of the evolutionary outcome. Generalization error considers the finite interaction and indicates that the size of the harness matters. Moreover, optimization error is closely linked to the specific evolutionary algorithm employed. The limitations of the algorithm itself will affect the quality of the solution. 
This decomposition 
provides the theoretical basis for the remainder of the paper, and we further use them to analyze the following proposed research questions: 

\emph{Q1 (harness architecture, approximation error).} 
Which user preferences can context harnesses reliably express, and when does additional computation or action-level control improve compliance? We compare harness mechanisms while fixing the model, actuator, and task distribution. 

\emph{Q2 (harness scale, generalization error).} 
How does the scale of a memory harness affect preference compliance when its contents are constructed from finite interaction evidence? We study this relationship through memory-length sweeps under a fixed model and preference set. 

\emph{Q3 (self-evolving algorithm, optimization error).} 
To what extent can repeated harness updates close the gap to an oracle context? We compare update recipes under a common feedback protocol and track their held-out performance across interaction checkpoints.

We study the proposed three research questions through complementary empirical and theoretical analyses. For each question, we first conduct controlled experiments on \emph{AppWorld-P}, a preference-oriented benchmark built on AppWorld~\citep{trivedi-etal-2024-appworld} that adds executable user-preference checks to the original tasks. The empirical analysis examines whether the phenomenon of interest arises in practice and characterizes when and how it occurs as we vary the harness mechanism, memory scale, or update procedure while keeping the foundation model fixed. We then analyze each question theoretically, identifying conditions and mechanisms that can explain the observed behavior under explicit modeling assumptions. This empirical-to-theoretical progression organizes~\cref{sec:main-results}: each subsection first presents the experimental evidence and then develops a theoretical interpretation through the lens of approximation, generalization, or optimization error. 

\section{Main Results}
\label{sec:main-results}
In this section, we present the main results for Q1---Q3 in~\cref{sec:approximation,sec:generalization,sec:optimization}, respectively. In each subsection, we first briefly introduce the problem settings and then present our empirical observations. Finally, we provide the theoretical analysis for understanding the empirical observations. A brief introduction to our experimental environment is as below. 

\paragraph{Benchmark and evaluation.} We conduct our experiments on AppWorld-P, a preference-oriented extension of AppWorld~\citep{trivedi-etal-2024-appworld}, an interactive environment in which agents perform everyday tasks through application APIs. Across all experiments, the foundation model is
\texttt{claude-haiku-4-5-20251001}~\citep{claude_haiku-4-5}. AppWorld-P preserves the original task environments and adds a preference-evaluation layer. Each user persona consists of executable rules specifying when a preference applies and how compliance is checked from the agent's actions and the resulting state. For example, a rule may require text messages to end with the user's first name or payment notes to follow a particular format. These rules make preference compliance programmatically measurable, while access to preference information is controlled by the experimental condition. For experiments involving harness evolution, the agent receives feedback after training interactions and updates its local harness; at each checkpoint, the harness is frozen and evaluated on held-out tasks. Static harness configurations are evaluated directly. This framework supports all three questions, with preferences and task pools tailored to each comparison. Our primary metric is the \emph{preference violation rate}: the fraction of applicable episode–rule pairs that violate the corresponding preference. Lower values indicate better compliance with the applicable preferences. Appendix~\ref{sec:detailed-exp-setting} provides full persona definitions, task splits, feedback channels, and implementation details, including the rule-specific checks and preference information available to each harness during interaction and evaluation.

\subsection{Approximation: On the Capability Boundaries of Context Harness}
\label{sec:approximation}
For Q1, we examine whether contextual specifications suffice for reliable preference compliance, or whether implementing a preference requires additional computation or action-level control.

\paragraph{Problem setting.}
Two preference groups probe the theoretical support distinction. \textbf{(i) In-Support:} local instructions the model is expected to follow once specified, including \emph{Private note format} (format payment notes using a prescribed template) and \emph{SMS sign-off} (end messages with the user's first name). \textbf{(ii) Out-of-Support:} preferences requiring history-based inference or exact computation, including \emph{SMS character checksum} (compute a character-count checksum for an outgoing message), \emph{Habitual card} (infer the user's habitual payment card from past interactions), and \emph{Running spend total} (maintain the cumulative amount paid to a recipient). 
Fixing the foundation model, constrained function-call actuator, and evaluation tasks, we compare five conditions: \textbf{(i) Context:} \emph{No memory} (no personalization context), \emph{Stated} (explicit golden preference description), and \emph{From history} (history-derived context); \textbf{(ii) Control:} \emph{Checker rejects} (a light-weight control implementation, verification and retry without computed answers) and \emph{Harness computes} (well designed control harness for specific preferences, including external computation or action rewriting). We report per-preference violation rates on applicable episodes. Detailed protocols are provided in Appendix~\ref{sec:detailed-Q1-setting}.

\begin{wrapfigure}{r}{0.5\textwidth}
    \centering
    \includegraphics[width=0.5\textwidth]{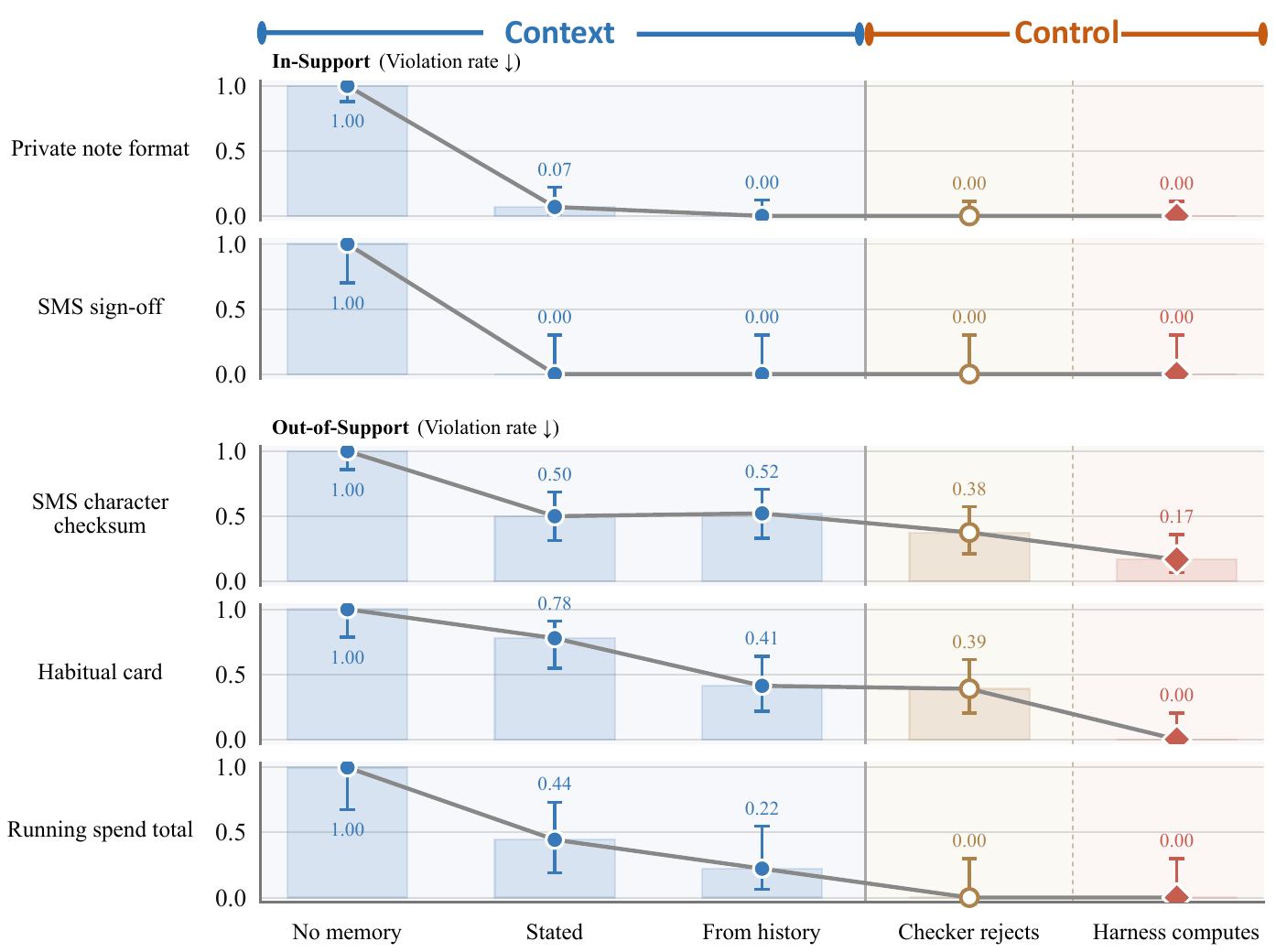}
    \caption{Preference violations by harness type. Context largely resolves in-support preferences; computation or action rewriting substantially improves compliance on out-of-support preferences.}
    \label{fig:q1-app}
\end{wrapfigure}

\paragraph{Empirical observation.}
The results show that explicit context is sufficient for the tested local instruction-following preferences, whereas history-dependent and computational preferences benefit substantially from additional harness mechanisms. 
As shown in Figure~\ref{fig:q1-app}, stating the preference reduces violations of \emph{Private note format} and \emph{SMS sign-off} from $1.00$ to $0.07$ and $0.00$, respectively. For \emph{SMS character checksum}, \emph{Habitual card}, and \emph{Running spend total}, stated context leaves rates of $0.50$, $0.78$, and $0.44$. History-derived context improves habitual-card inference and reduces running-total violations to $0.22$, without reliably resolving all three preferences. Lightweight \emph{Checker rejects} achieves zero observed running-total violations through verification and retry, and also reduces violation rates to $0.38$ on checksum and $0.39$ on habitual card. \emph{Harness computes} further reduces these two rates to $0.17$ and $0.00$, while also achieving $0.00$ on running totals. Thus, Q1 calls for matching the mechanism to the preference: textual specifications can activate familiar behaviors, but compliant executions and computational preferences need a control-based guidance. We next analyze how these mechanisms shape reachable policies through a latent-concept model. 

\paragraph{Theoretical analysis.}
We analyze this separation through a latent-concept model of the frozen model's reachable policies. Let $P_\theta:=p(\cdot\mid\cdot,\theta)$ be the trace policy for behavioral concept $\theta$, and $\mathcal M_f:=\overline{\operatorname{conv}}\{P_\theta:\theta\in\operatorname{supp}p(\theta)\}$ the closed convex hull of policies supported by the pretrained prior $p(\theta)$. Under Assumption~\ref{assump:context_latent}, context reweights these concepts independently of the test task, while leaving each $P_\theta$ fixed. The context- and control-reachable sets are $\mathcal R_{\mathrm{ctx}}:=\{\pi_{c\circ f}:c\in\mathcal C_{\mathrm{ctx}}\}$ and $\mathcal R_{\mathrm{ctrl}}:=\{\mathcal T_c\pi_f:c\in\mathcal C_{\mathrm{ctrl}}\}$, where $\mathcal T_c$ transforms model-proposed traces into executed traces, for example by computing an action argument or rewriting a preference-owned field. The following theorem links these reachable sets to approximation error, identifying when control can attain the user optimum while the context class retains a strictly positive risk gap (proof in Appendix~\ref{app:proof-q1}).


\begin{theorem}[Approximation gap of context and control harnesses]\label{thm:q1}
    Under Assumption~\ref{assump:context_latent}, $\mathcal R_{\mathrm{ctx}}\subseteq\mathcal M_f$, whereas $\mathcal R_{\mathrm{ctrl}}$ need not be contained in $\mathcal M_f$. If the user-optimal policy $\pi_{f_u^*}\in\mathcal R_{\mathrm{ctrl}}\setminus\mathcal R_{\mathrm{ctx}}$ 
    and $\Delta_u:=\inf_{\pi\in\mathcal R_{\mathrm{ctx}}}R(\pi)-R(\pi_{f_u^*})>0$, then
        $\epsilon_{\mathrm{app}}(\mathcal C_{\mathrm{ctrl}},f_u^*)=0
        <\epsilon_{\mathrm{app}}(\mathcal C_{\mathrm{ctx}},f_u^*)=\Delta_u.$
\end{theorem}

\emph{Messages from Theorem~\ref{thm:q1}.}
The theorem answers Q1 through policy reachability: \emph{context} selects among existing behaviors, while \emph{control} can transform their execution. Under the stated risk separation, refining context alone cannot eliminate the approximation gap, whereas a suitable control harness can reach the user-optimal policy. This perspective explains the contrast in Figure~\ref{fig:q1-app}: in-support formatting and sign-off requirements activate familiar instruction-following behaviors, while out-of-support habitual-card inference and running totals benefit from maintained statistics, and checksums from explicit computation. For harness design, the implication is to match the mechanism to the preference: use context to express supported behaviors, and introduce targeted state tracking, computation, or action rewriting where reliable execution requires these operations.


\subsection{Generalization: On the Impact of Harness Scale}
\label{sec:generalization}
Having compared harness mechanisms, we turn to Q2: whether adding more learned statements to a context harness consistently improves preference compliance under finite interaction evidence.

\paragraph{Problem setting.}
We use eight preferences from three applications: \textbf{(i) SMS:} \emph{SMS greeting} (start with a greeting), \emph{SMS sign-off} (end with the user's first name), and \emph{SMS terseness} (use at most five words); \textbf{(ii) Venmo:} \emph{Venmo private} (mark transactions private), \emph{Payment has note} (include a non-empty payment note), \emph{Payment-note lowercase} (write notes entirely in lowercase), and \emph{Single-word payment note} (use exactly one word per note); \textbf{(iii) Spotify:} \emph{Playlist over like} (save songs to playlists rather than liking them individually). Memory blocks draw from a fixed pool of statements learned from interaction trajectories and instance-level complaints, rather than hand-written oracle specifications. We select statements round-robin across preference groups, reusing entries for longer blocks; all concern scored preferences, with no unrelated content added. Holding the foundation model, actuator, preference set, and evaluation tasks fixed, we test nine lengths, $L\in\{0,1,2,3,5,10,20,60,150\}$, including the no-memory reference $L=0$. Each block is shuffled separately for each $(L,\mathrm{seed})$ and frozen for evaluation over three seeds, 12 tasks, and 10 rollouts per task. We report eight-rule violation rates and assess scoring-set sensitivity by re-scoring the same episodes from the full eight-preference harness on six rules: we exclude \emph{SMS terseness}, which can conflict with task-required message content and formatting, and \emph{Payment has note}, which is implied by \emph{Single-word payment note}. See Appendix~\ref{sec:detailed-Q2-setting} for memory-construction and evaluation details.

\begin{wrapfigure}{r}{0.5\textwidth}
    \centering
    \includegraphics[width=0.5\textwidth]{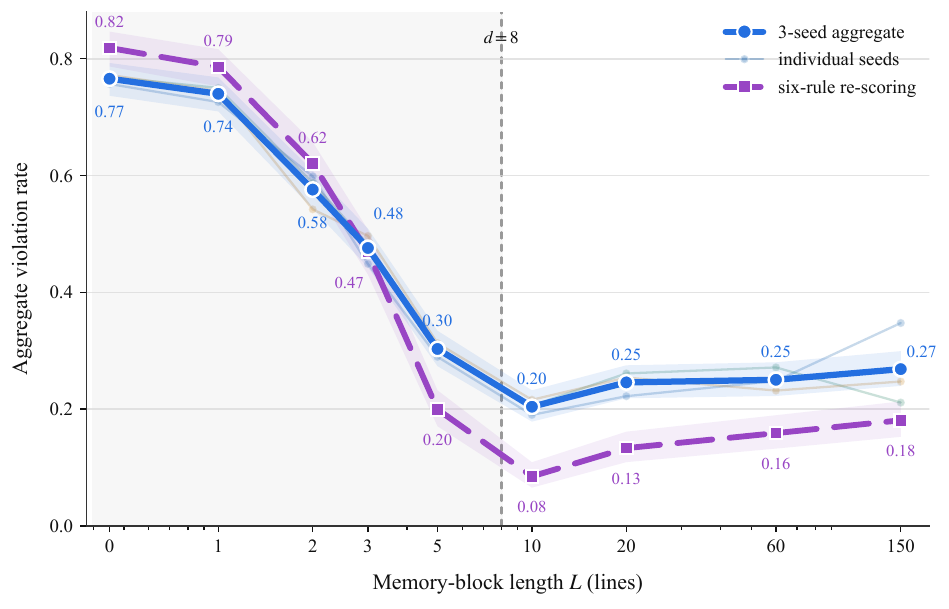}
    \caption{Non-monotonic violations across memory lengths. Solid: three-seed pooling over eight rules; faint: individual seeds; dashed: the same evaluation episodes re-scored on six preferences.}

    \label{fig:q2-gen}
\end{wrapfigure}

\paragraph{Empirical observation.}
Increasing memory scale improves preference compliance initially, but does not produce sustained gains under the evaluated memory construction. Figure~\ref{fig:q2-gen} shows that the pooled violation rate decreases from approximately $0.77$ at $L=0$ to $0.20$ at $L=10$, consistent with improved coverage as more preferences are represented in memory. Further expansion reverses part of this improvement: violations rise to approximately $0.25$--$0.27$ at $L=20$--$150$, even though every injected statement concerns a scored preference. The six-rule re-scoring retains the same overall decline-and-rise pattern at lower absolute rates, indicating that the pattern is not solely an artifact of including the two excluded rules in the aggregate metric. These results answer Q2 by showing that a larger collection of relevant, learner-generated statements is not necessarily a more effective harness. 
The following theoretical analysis examines the competing contributions of preference coverage and estimation error under finite interaction, providing a stylized account of why increasing memory scale can have both benefits and costs. 

\paragraph{Theoretical analysis.}
To explain the observed reversal in performance, we combine an information-theoretic memory model with concentration bounds for noisy feedback. Consider $d$ independent, unbiased binary preferences, evaluated uniformly and executable once correctly specified. Let $\mathcal C_L$ contain context harnesses with at most $L$ statements, with population optimum $c_{\mathcal C_L}^*$. The learned harness $\hat c_{n,L}$ minimizes statement-wise feedback disagreement by majority vote, using $m=\lfloor n/L\rfloor$ observations per statement from a budget of $n$ observations. We assume that recovering all target statements reproduces class-optimal preference behavior, linking statement-inference errors to the generalization gap. The following theorem bounds the resulting generalization error and shows how this bound increases with memory length $L$ when the feedback budget $n$ is fixed (detailed proofs are presented in Appendices~\ref{app:proof-q2}--\ref{app:proof-q2-cor}).


\begin{theorem}[Generalization cost of harness scale]\label{thm:q2}
    Under the above model, let $1\leq L\leq n$ and suppose each feedback observation independently reports its associated preference correctly with probability $1/2+\gamma$, where $0<\gamma\leq1/2$. Then
    \[
        \mathbb E_{S_n}\!\left[\epsilon_{\mathrm{gen}}(\mathcal C_L,n)\right]
        \leq L\exp\!\left(-2\gamma^2\left\lfloor\frac{n}{L}\right\rfloor\right).
    \]
\end{theorem}

For fixed $n$ and $\gamma$, this generalization bound is strictly increasing in integer $L$. Storing $\min\{L,d\}$ correct preferences gives the coverage bound $\epsilon_{\mathrm{app}}(\mathcal C_L,f_u^*)\leq(d-L)_+/d$, where $(x)_+:=\max\{x,0\}$. Combining the two bounds yields the following excess-risk guarantee.

\begin{corollary}[Coverage--estimation tradeoff]\label{cor:q2-excess-risk}
    Under the same model and feedback assumptions,
    \[
        \mathbb E_{S_n, U}\!\left[R(\pi_{\hat c_{n,L}\circ f})-R(\pi_{f_u^*})\right]
        \leq\frac{(d-L)_+}{d}+L\exp\!\left(-2\gamma^2\left\lfloor\frac{n}{L}\right\rfloor\right).
    \]
\end{corollary}

\emph{Messages from Theorem~\ref{thm:q2} and Corollary~\ref{cor:q2-excess-risk}.}
Theorem~\ref{thm:q2} identifies a growing statistical burden behind memory expansion: a fixed feedback budget must support more statements, leaving fewer observations per statement and more opportunities for incorrect inference. Corollary~\ref{cor:q2-excess-risk} combines this increasing estimation bound with a decreasing coverage bound. Once coverage saturates, further expansion no longer reduces the approximation bound, while the estimation bound continues to increase. These competing effects can yield a U-shaped upper bound, consistent with the initial improvement and subsequent degradation in Figure~\ref{fig:q2-gen}. Q2 therefore reflects a coverage--estimation tradeoff: relevant statements expand what a harness can represent, but their reliability depends on the evidence supporting them. For harness design, memory growth should track both the amount and reliability of feedback: retain statements that add supported preference coverage, consolidate redundant descriptions, and gather further evidence before expanding uncertain parts of the memory.

\subsection{Optimization: On the Limits of Self-Evolving Harness}
\label{sec:optimization}
Beyond harness architecture and scale, Q3 concerns whether repeated feedback-driven updates approach oracle compliance and how this progress depends on the update procedure.

\begin{wrapfigure}{r}{0.5\textwidth}
    \centering
    \includegraphics[width=0.5\textwidth]{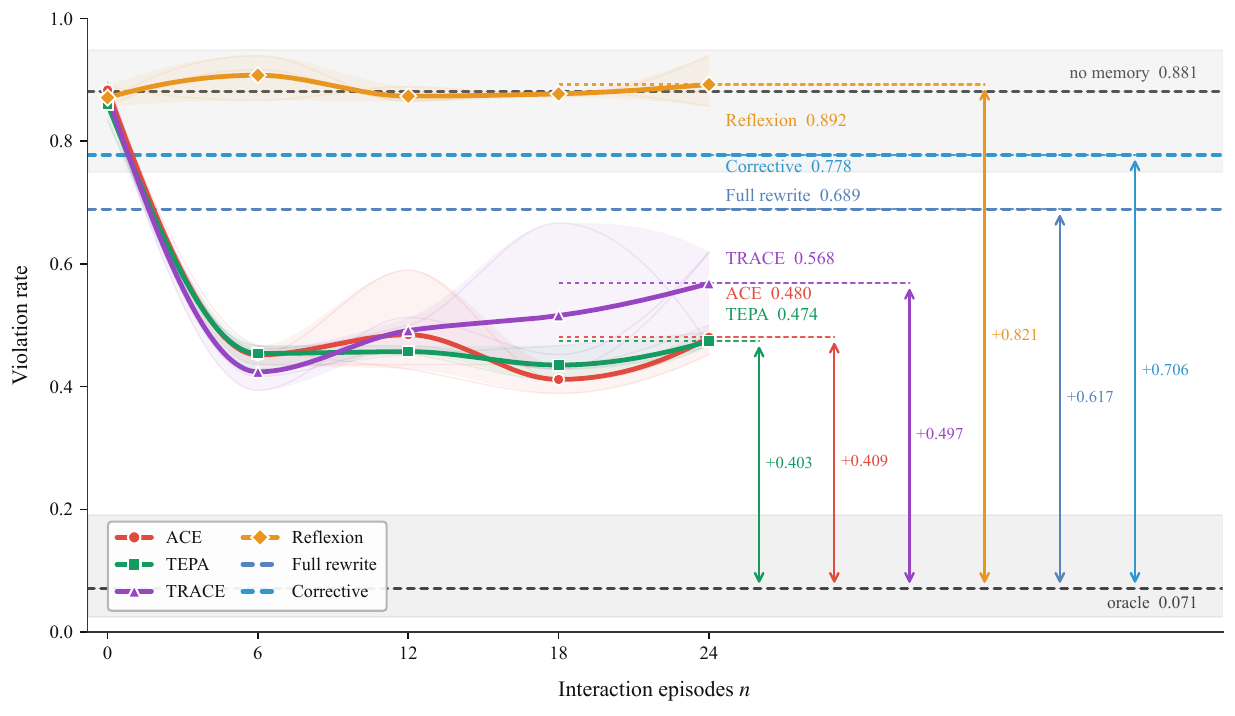}
    \caption{Self-evolution trajectories and remaining oracle gaps. Curves show three-seed means with min--max bands for the self-evolving harnesses. }

    \label{fig:q3-opt}
\end{wrapfigure}

\paragraph{Problem setting.} 
We fix a persona with five preferences: \emph{SMS greeting} (begin messages with a greeting word), \emph{SMS sign-off} (end with the user's first name), \emph{Venmo private} (mark transactions private), \emph{Payment-note category} (begin notes with the category tag), and \emph{Payment-note initials} (end notes with parenthesized dotted initials). Training feedback consists of instance-level complaints that specify the desired behavior for the first three preferences but identify only the defective field for the latter two, withholding the required category tag or exact initials format. This tests the use of available feedback. We compare \textbf{(i) Static references:} \emph{No memory} (no personalization memory) and \emph{Oracle context} (explicit descriptions of all five preferences); and \textbf{(ii) Self-evolving harnesses:} adaptations of ACE~\citep{ICLR2026_ACE} (incrementally updated itemized memory), TEPA~\citep{zhou2026teparevokingstalememories} (keyed precedents with replacement and revocation), TRACE~\citep{zhou2026gettingbetterworkingyou} (learned runtime checks with retries), and Reflexion~\citep{NEURIPS2023_Reflexion} (a bounded buffer of verbal reflections). Diagnostic arms use full-memory rewriting with instance-level or fully specified corrective feedback. With the foundation model, actuator, persona, and feedback protocol fixed, each of the four main methods updates from its own outcomes over 24 training interactions and three seeds. At $n\in\{0,6,12,18,24\}$, we freeze each harness and evaluate it on six held-out tasks with three rollouts per task. We compare self-evolving harnesses with the static references to characterize improvement and the remaining oracle gap. Detailed protocols are provided in Appendix~\ref{sec:detailed-Q3-setting}.

\paragraph{Empirical observation.}
Repeated harness updates narrow the oracle gap for some recipes, but none of the tested procedures closes it within the evaluated interaction budget. As shown in Figure~\ref{fig:q3-opt}, ACE, TEPA, and TRACE reduce violations to approximately $0.4$--$0.5$ after six interactions, compared with the no-memory reference of $0.881$. Their subsequent trajectories differ. ACE and TEPA fluctuate around similar levels and finish near $0.48$, leaving a gap of approximately $0.40$ to the oracle rate of $0.071$. TRACE loses part of its early improvement and ends near $0.57$, while Reflexion remains close to the no-memory reference throughout. The full-rewrite and corrective-feedback diagnostics also finish above the oracle, at approximately $0.69$ and $0.78$, respectively. The answer to Q3 is therefore that self-evolution can improve personalization, but additional interactions do not necessarily translate into continued progress toward oracle compliance. The similar ACE and TEPA endpoints suggest that the persistent gap is not specific to one memory-update recipe, while the other trajectories demonstrate the importance of the update procedure. Because the feedback withholds some target details, this gap must also be interpreted in light of the information available to the learner. We next examine conditions under which biased update dynamics can sustain residual error, providing a theoretical perspective on the limits of repeated self-evolution.

\paragraph{Theoretical analysis.}
We interpret these trajectories through a time-homogeneous Markov model of online harness updates and a Wasserstein fixed-point analysis. For a fixed model $f$, capacity $L$, and feedback protocol, let $U$ denote the update rule. Write $\nu_{n,L}$ and $\widehat\nu_{n,L}$ for the distributions of the deployed harness $\tilde c_{n,L}$ and empirical minimizer $\hat c_{n,L}$. 
The operator $\mathcal T_U$ maps a harness distribution to its distribution after one interaction and update. With $W_L$ denoting the first Wasserstein distance under a harness metric $\mathsf d_L$, define $\beta_{U,n,L}:=W_L(\mathcal T_U\widehat\nu_{n,L},\widehat\nu_{n,L})$, the displacement induced by updating the distribution of empirical minimizers. Assumptions~\ref{ass:q3-dynamics}--\ref{ass:q3-bias} quantify contraction by $\rho<1$, quadratic empirical-risk growth by $\mu>0$, and empirical-to-population estimation risk by $\epsilon_{n,L}\to0$ as $n\to\infty$. The following theorem links a persistent update residual to a positive optimization floor (a detailed proof is provided in Appendix~\ref{app:proof-q3}).


\begin{theorem}[Persistent optimization floor under biased updates]\label{thm:q3}
    Under Assumptions~\ref{ass:q3-dynamics}--\ref{ass:q3-bias}, $\nu_{n,L}$ converges to a unique invariant distribution and its expected risk converges, while
    \begin{align*}
    \liminf_{n\to\infty}\mathbb E\!\left[\epsilon_{\mathrm{opt}}(U;\mathcal C_L,S_n)\right]
    \geq\frac{\mu\beta_0^2}{2(1+\rho)^2}>0.
    \end{align*}
    If a family of update rules satisfies these assumptions with common constants $\rho,\mu,\beta_0$, the same positive lower bound holds for every rule in the family.
\end{theorem}

\emph{Messages from Theorem~\ref{thm:q3}.}
The theorem answers Q3 by showing how online harness updates can stabilize while retaining a positive optimization gap. Even as risk-estimation error vanishes, continued iterations need not escape a suboptimal stationary regime. This obstruction can arise under different update recipes satisfying the stated assumptions, while allowing different trajectories and limiting risks. Such dynamics provide a mechanism consistent with the performance plateaus and persistent oracle gaps in Figure~\ref{fig:q3-opt}. The lower bound highlights how changing the update recipe can leave a common source of error unresolved, revealing the inherent bottleneck of self-evolving algorithms. Harness design should therefore target these sources of update bias through more informative feedback, grounded correction signals, and independently validated memory edits, aligning the update dynamics with the preference objective.

\section{Conclusion}
\label{sec:conclusion}
We studied how harness design and evolution shape personalization around a frozen foundation model. Experiments on AppWorld-P reveal that preference compliance depends on the harness mechanisms available, does not improve monotonically with memory size, and can remain below oracle performance despite repeated updates. Our learning-theoretic formulation connects these observations to distinct limitations in approximation, generalization, and optimization, providing conditional explanations for when and why adaptation falls short. These findings suggest that reliable personalization requires jointly considering what a harness can express, how its capacity matches the available evidence, and how feedback guides its updates. This perspective provides a unified account of diverse harness-engineering practices by identifying which error components they seek to reduce, supporting both failure diagnosis and more principled design in future research. 

\section*{AI use statement}

We used OpenAI Codex and Anthropic Claude Code to assist with core code implementation, data construction and processing, evaluation, interpretation of results, translation, and the checking and refinement of mathematical statements and proofs. Additionally, we used these tools for literature search and synthesis, manuscript drafting and editing, and figure preparation. The initial research ideas, conceptual framework, hypotheses, proof strategies, and experimental design were developed by the authors without AI assistance. The authors reviewed all AI-assisted work, including the text, references, mathematical arguments, code, data and evaluation procedures, and reported results and figures. We take full responsibility for the final content of this work, including all text, claims, and artifacts produced with the assistance of generative AI. 

\section*{Ethics Statement}

This work studies personal-agent adaptation in the simulated application environments of AppWorld-P using explicitly specified preference rules. The evaluated messaging, payment, and other application actions take place within the benchmark environment. Our aim is to understand limitations that can undermine reliable personalization. In practical deployments, persistent memories may contain sensitive information, and incorrectly inferred preferences may lead to unintended actions. Responsible deployment therefore requires user consent, protection of stored preferences and interaction logs, and appropriate oversight of consequential actions. Our findings motivate assessing preference compliance alongside privacy, safety, and user control.

\section*{Reproducibility Statement}

Section~\ref{sec:main-results} describes the benchmark, evaluation metric, and experimental comparisons. Appendix~\ref{sec:detailed-exp-setting} details the preference rules, task selection, memory construction, feedback channels, model configuration, seeds, and training and evaluation protocols. The assumptions and proofs supporting our theoretical claims are provided in Appendices~\ref{app:proof-q1}--\ref{app:proof-q3}. The code is open-sourced via an anonymous GitHub repository listed at the end of the abstract. 



\bibliography{reference}
\bibliographystyle{iclr2027_conference}

\appendix
\section{Detailed Experimental Settings}
\label{sec:detailed-exp-setting}
This section provides the complete experimental protocol underlying the three main questions. All experiments use AppWorld-P, our preference-evaluation layer over AppWorld~\citep{trivedi-etal-2024-appworld}. AppWorld itself is left unchanged: each task is executed in its original application world, while AppWorld-P attaches a persona consisting of programmatic preference rules. A rule specifies (i) when it is applicable, (ii) whether the recorded API calls satisfy the preference, and (iii) the feedback exposed to the harness after a training episode. Rule semantics and checker state are never exposed to the agent except through the arm-specific channel described below. 

Across all experiments, the foundation model is
\texttt{claude-haiku-4-5-20251001}~\citep{claude_haiku-4-5}. We use a constrained function-call actuator: at each agent step, the model may emit exactly one API call with literal arguments. This restriction prevents the model from outsourcing arithmetic or aggregation to an interpreter and makes the distinction between contextual and computational harnesses well defined. Unless stated otherwise, the model temperature is $0.7$, the per-task step budget is 50, and evaluation uses a frozen task set. The primary metric is the micro-averaged preference violation rate,
\begin{equation}
    \widehat V
    =\frac{\sum_{i,r}\mathbf{1}[r\text{ is applicable in episode }i]
        \mathbf{1}[r\text{ is violated in episode }i]}
      {\sum_{i,r}\mathbf{1}[r\text{ is applicable in episode }i]}.
    \label{eq:appendix-violation-rate}
\end{equation}
Episodes in which the constrained action is not taken are not counted as either a success or a violation for that rule. 
Where confidence intervals are reported, we use Wilson 95\% intervals. 

Concretely, a persona is a small declarative bundle of rules rather than a free-form prompt. Each rule names the API action it constrains, the field or state variable it checks, and a natural-language statement used only by the corresponding context arm. For example, the Q1 persona can be represented as: 
\begin{quote}
\small\texttt{name: q1\_format\_checksum}\\
\texttt{rules: [private\_note\_format, sms\_char\_checksum]}\\
\texttt{private\_note\_format.trigger: venmo.create\_transaction}\\
\texttt{sms\_char\_checksum.trigger: phone.send\_text\_message}
\end{quote}
The checker decides whether each rule is applicable from the recorded API calls and then evaluates the resulting note or message. Thus the same persona can be attached to different harness arms without changing the task world or the model. 

\subsection{Detailed Experimental Settings for Q1}
\label{sec:detailed-Q1-setting}
\paragraph{Objective and preference selection.}
Q1 tests whether changing the harness implementation changes the set of preferences that the fixed model can satisfy. We first calibrate candidate rules using two static conditions: no memory and a single, maximally explicit natural-language description of the preference. A rule is labeled \emph{in-support} when the stated description yields near-perfect compliance and the no-memory condition does not; it is labeled \emph{out-of-support} when a substantial violation rate remains despite the perfect description. These labels are determined empirically rather than from the rule's source file or our prior judgment.

The final experiment contains five preferences (Table~\ref{tab:q1-rules}). The two in-support rules test formatting and sign-off instructions that the model can implement once stated. The three out-of-support rules cover complementary failure modes: exact computation on the current action, statistical inference from an interaction history, and exact aggregation across both past and within-episode actions.

\begin{table}[h]
\centering
\caption{Preferences used in Q1. The family assignment is determined by the
perfect-description calibration.}
\label{tab:q1-rules}
\resizebox{\hsize}{!}{
\begin{tabular}{llll}
\hline
\textbf{Preference}    & \textbf{Family} & \textbf{Trigger} & \textbf{Requirement}                                                                                                                                       \\ \hline
Private note format    & In-support      & Venmo payment    & \begin{tabular}[c]{@{}l@{}}The note follows a private bracketed format; its code is determined\\ by the payment amount.\end{tabular}                       \\
\rowcolor[HTML]{EFEFEF} 
SMS sign-off           & In-support      & Text message     & The message ends with the user’s first name.                                                                                                               \\
SMS character checksum & Out-of-support  & Text message     & \begin{tabular}[c]{@{}l@{}}The message ends with a checksum equal to its character count\\ modulo seven.\end{tabular}                                      \\
\rowcolor[HTML]{EFEFEF} 
Habitual card          & Out-of-support  & Venmo payment    & \begin{tabular}[c]{@{}l@{}}The first attempted payment card is the user’s latent habitual bank\\ card.\end{tabular}                                        \\
Running spend total    & Out-of-support  & Venmo payment    & \begin{tabular}[c]{@{}l@{}}The note ends with the cumulative amount successfully sent to that\\ recipient, including the current payment and earlier successful ones.\end{tabular} \\ \hline
\end{tabular}
}
\end{table}

\paragraph{Harness conditions.}
The five columns in Figure~\ref{fig:q1-app} compare contextual specifications with two forms of execution control:
\begin{enumerate}
    \item \textbf{No memory} injects no personalization context and provides the unpersonalized reference.
    \item \textbf{Stated} injects the rule's clearest natural-language description. This removes preference-induction error and provides the principal context-class
    reachability reference.
    \item \textbf{From history} asks the model to use interaction-derived context. For the four rules in the main AppWorld streams, a top-$L$ assertion updater selects three statements from a bank containing the two persona rules and six distractors, using corrective feedback over the training stream specified below. For the habitual-card rule, this condition contains the adoptable rule statement and the complete raw habit log; no statistic is precomputed.
    \item \textbf{Checker rejects} combines the stated rule with up to three whole-episode attempts. A programmatic checker accepts the first attempt with an applicable action and no preference violation; otherwise it supplies rejection feedback for the next attempt. This channel restates the constraint without returning the computed target value. For habitual card, it also identifies the rejected card and accumulates exclusions across attempts. 
    \item \textbf{Harness computes} supplies preference-specific external computation. For character checksum, the checker returns the expected value for the previous message to the model, with up to three fresh attempts; it does not directly rewrite the suffix. The habitual-card harness maintains per-card acceptance rates and injects their argmax. The running-total harness maintains a ledger, writes the computed total into the API's actual \texttt{description} field before execution, and updates the ledger only after successful payment. For the two in-support rules, the value-returning configuration adds no target computation.
\end{enumerate}
The first three conditions are context harnesses. The last two are control harnesses: \emph{Checker rejects} adds a lightweight verification-and-retry gate, whereas \emph{Harness computes} supplies derived values or rewrites action arguments. Although rejection feedback is textual, the external checker also controls retry and trajectory selection. These conditions compare mechanisms, not equal inference budgets.

The stated arm inserts the rule's \texttt{oracle\_text} verbatim. A representative prompt fragment is:
\begin{promptbox}{Stated Harness: Oracle Preference Context}
My preferences for how you should do things:\\
1. Venmo payment descriptions must be exactly [<code>|<date>|<summary>] where <code> is K7 if the whole-dollar part of the amount is even and Q3 if it is odd, <date> is today's date as YYYY-MM-DD, and <summary> is at most 6 words.\\
2. Text messages you send for me must end with my first name as a sign-off.
\end{promptbox}
The reject-only arm instead returns a bounded failure message. For a habitual-card attempt, an illustrative instance is:
\begin{examplebox}{Checker-Rejection Example}
\texttt{Automated constraint checker REJECTED the previous attempt.}\\
\texttt{You paid with the Chase card, which is not the one -- pick a different card next time.}\\
\texttt{Redo the task so the constraint passes.}
\end{examplebox}
This reports the card just attempted, but never the latent habitual bank. Other rules use the same reject-only principle: the checker may identify the failed attempt or restate the constraint, but it does not return the value that must be computed.

\paragraph{Personas, streams, and sample sizes.}
The five panels combine three experiments because several rules claim incompatible fields and cannot coexist in one persona. The main persona contains private note format and SMS character checksum, with eight training tasks, seven frozen evaluation tasks, six seeds, and a 50-step budget per attempt. The crossover persona contains SMS sign-off and running spend total, reversing the easy/hard assignment across the two applications; it uses six training tasks, six frozen evaluation tasks, three seeds, and a 60-step budget per attempt. Static arms are evaluated once per seed; history-derived results use the final checkpoint, at $n=8$ and $n=6$, respectively.
The habitual-card experiment uses a synthetic stream of binary-feedback micro-interactions followed by six fixed AppWorld evaluation tasks. Each micro-interaction records the card used and whether the user accepted it, but not which card the user would have preferred. The latent preference is a categorical distribution whose mode is the habitual card. We evaluate after 0, 20, 60, and 120 observations and rotate the mode across different candidates. The headline figure reports $n=20$, while the full schedule diagnoses the context method's sample efficiency. The per-attempt step budget is 30.


\paragraph{Scoring and controls.}
All arms within each experiment see the same evaluation tasks; applicable-episode counts can differ with the actions actually executed. An attempt with no applicable persona rule is not accepted by the verifier. If its retry budget is exhausted, the final attempt is retained unless it contains no applicable action, in which case the last acted attempt is preferred. For habitual card, the first card attempted is scored, since later cards may reflect an insufficient-balance fallback. The non-leaking reject-only arm is also compared with a card-cycling elimination baseline over the 4--5 available cards. For running total, the checker and computational harness start from the same read-only transaction ledger. Only successful payments contribute to the cumulative amount or make this rule applicable; failed calls do neither. Each successful payment is checked against the cumulative total including that payment, using the actual \texttt{description} field and validating the recorded outcome against the transaction identifier and database. For checksum, returning the expected value still leaves the model responsible for composing a valid final message. 

\subsection{Detailed Experimental Settings for Q2}
\label{sec:detailed-Q2-setting}
\paragraph{Objective and scan set.}
Q2 studies the relationship between harness scale and preference compliance by varying the number $L$ of injected memory lines. The adopted persona contains eight preferences spanning text messages, Venmo payments, and Spotify actions: SMS sign-off, SMS greeting, SMS terseness, payment-note presence, payment-note lowercasing, single-word payment notes, private Venmo transactions, and saving songs to a playlist rather than liking them individually. The details are listed in Table~\ref{tab:q2-rules}. Every rule is represented by text induced by the same learner from instance-level feedback; the memory is not constructed from hand-written oracle statements. 
The eight-rule set was chosen using induced-text fidelity: a candidate was retained only when the learner-generated statement described a behavior that could, in principle, affect the corresponding checker. 

\begin{table}[h]
\centering
\small
\caption{Preferences administered in the Q2 memory-length sweep. All eight rules are included in the primary violation metric.}
\label{tab:q2-rules}
\resizebox{\hsize}{!}{
\begin{tabular}{lll}
\hline
\textbf{Preference} & \textbf{Trigger} & \textbf{Requirement} \\ \hline
SMS sign-off & Text message & End with the user's first name. \\
\rowcolor[HTML]{EFEFEF}
SMS greeting & Text message & Begin with a greeting word. \\
SMS terseness & Text message & Use at most five words. \\
\rowcolor[HTML]{EFEFEF}
Payment has note & Venmo payment & Include a non-empty description. \\
Payment-note lowercase & Venmo payment & Write the note entirely in lowercase. \\
\rowcolor[HTML]{EFEFEF}
Single-word payment note & Venmo payment & Use exactly one word. \\
Venmo private & Venmo payment & Mark the transaction private. \\
\rowcolor[HTML]{EFEFEF}
Playlist over like & Spotify save & Add songs to a playlist instead of liking individually. \\ \hline
\end{tabular}
}
\end{table}

\paragraph{Learner-generated assertion pool.}
Before the length sweep, we collect assertions by running the learner for 32 episodes using instance-level complaints. During collection, the injected memory is empty, so every assertion is a function of the observed task trajectory and feedback rather than an earlier memory. We collect two independent pool seeds, deduplicate exact repeats, group assertions by their generating rule, and rank within each group by consensus with other learner paraphrases. Consensus ranking uses only the generated text and never the oracle rule statement. A round-robin construction then cycles across rule groups, which makes coverage increase as evenly as possible as $L$ grows. 

The collection prompt is instantiated once per violated rule. It receives the task, the write actions, and the episode-specific complaint, but not the rule name or its oracle statement:
\begin{promptbox}{Assertion-Pool Induction Prompt}
You maintain a long-term memory of one user's personal preferences for how their assistant should do things. The user has just complained about something the assistant did. From that complaint, write the standing preference you will store in memory. \\
Task the user asked for: <instruction>\\
What the assistant did: <API writes>\\
What the user said about it: <complaint>\\
Output exactly one sentence as a general user preference; output only the sentence.
\end{promptbox}
For example, the complaint \texttt{sms not signed with first name: 'Please get on venmo.' -- I don't want it done that way.} produced the learner assertion ``Text messages should be signed with your first name.'' in one collection seed. A different seed produced ``When sending text messages on behalf of the user, they should be signed with the user's first name.'' These paraphrases remain separate pool entries and are later selected by the length sweep.

\paragraph{Primary length sweep.}
The primary R-arm contains only assertions about the eight scored preferences. We sweep
\begin{equation}
    L\in\{0,1,2,3,5,10,20,60,150\}.
\end{equation}
When $L$ exceeds the number of distinct learner assertions, the finite pool is recycled. 
For each $(L,\text{seed})$ cell, the selected lines are shuffled deterministically using the evaluation seed and $L$. 

We run three evaluation seeds. Each cell contains 12 frozen tasks with 10 fresh rollouts per task, for 120 episodes per cell and 3,240 episodes over the $9\times3$ primary grid. The task pool is balanced so that each of the eight rules has four triggering tasks per cell before behavioral inapplicability. 

\subsection{Detailed Experimental Settings for Q3}
\label{sec:detailed-Q3-setting}
\paragraph{Objective and persona construction.}
Q3 tests whether a self-evolving harness reaches the oracle context when the model, actuator, task distribution, and preference set are fixed. We use a persona containing five jointly satisfiable preferences over phone and Venmo actions (Table~\ref{tab:q3-rules}). The set is designed to separate preferences whose target is carried by the feedback channel from those whose target literal is withheld. This distinction is checked mechanically by executing each rule and comparing the quoted literals in its oracle statement with the checker detail available to the updater.

\begin{table}[h]
\centering
\small
\caption{Preferences used in Q3 and whether instance-level feedback carries the target
needed to reconstruct the oracle statement.}
\label{tab:q3-rules}
\resizebox{\hsize}{!}{
\begin{tabular}{lll}
\hline
\textbf{Preference}   & \textbf{Feedback role} & \textbf{Requirement}                                                                   \\ \hline
SMS greeting          & Carried                & Begin with a greeting word.                          \\
\rowcolor[HTML]{EFEFEF} 
SMS sign-off          & Carried                & End a text message with the user’s first name.                                         \\
Venmo private         & Carried                & Mark a Venmo transaction private.                                                      \\
\rowcolor[HTML]{EFEFEF} 
Payment-note category & Target withheld        & Begin the note with the literal category tag \texttt{{[}personal{]}}. \\
Payment-note initials & Target withheld        & End the note with parenthesized initials such as \texttt{(J.D.)}.     \\ \hline
\end{tabular}
}
\end{table}
The category and initials complaints identify the defective field but never reveal the literals \texttt{[personal]} or \texttt{(J.D.)}. By contrast, the two SMS complaints identify the required form. 
Joint satisfiability is verified by executing all five checkers on a constructed episode, not only by comparing their prose descriptions. 

For example, one training episode attempted the SMS ``Done.'' and a Venmo note
``Grocery Bill.'' Its instance-level feedback was:
\begin{examplebox}{Instance-Level Feedback Example}
\texttt{That's not how I like things done:}\\
\texttt{- note lacks the category tag: 'Grocery Bill' -- I don't want it done that way.}\\
\texttt{- note lacks initials suffix: 'Grocery Bill' -- I don't want it done that way.}\\
\texttt{- venmo transaction not marked private -- I don't want it done that way.}\\
\texttt{- sms lacks an opening greeting: 'Done.' -- I don't want it done that way.}\\
\texttt{- sms not signed with first name: 'Done.' -- I don't want it done that way.}
\end{examplebox}
The updater sees this text, while the correspondence between each complaint line and its checker rule is retained only as driver-side provenance for scoring and analysis. In particular, the first two complaints identify the defective note field but do not disclose the required category tag or the parenthesized initials format. 

\paragraph{Interaction stream and evaluation split.}
Each evolving run contains 24 on-policy training episodes. The current harness is injected into the agent at the start of an episode; after task execution, the programmatic checkers produce instance-level feedback, the updater modifies the harness, and only the next episode observes the new state. The main checkpoints are $n\in\{0,6,12,18,24\}$. At each checkpoint, the harness is frozen and evaluated on six tasks with three independent rollouts per task, yielding 18 evaluation episodes. 
We use three seeds. 
The training and evaluation splits are disjoint at the task-template-family level. 
This design tests whether an induced preference transfers across task templates rather than whether the updater memorizes a specific task. 

\paragraph{References and update recipes.}
The no-memory arm provides the initial behavioral ceiling. The oracle arm injects the five clearest rule statements. It is the reachability reference against which the final self-evolution gap is measured. The primary evolving method is an \textbf{ACE}-style updater with itemized bullets, model-proposed add/modify deltas, deterministic code-side merging, trigram-based semantic deduplication, and a capacity of at most 12 bullets. 
We compare the primary trajectory with three admissible self-evolution recipes under the same supervision: \textbf{TEPA} stores keyed precedents, replaces the active precedent under the same key, and retains revoked entries for audit; \textbf{TRACE} compiles feedback into constrained runtime predicates and permits up to three self-gated attempts; and \textbf{Reflexion} appends natural-language reflections to a FIFO buffer of width three. TRACE gates against its own learned checks, never the ground-truth persona checker, so it receives no privileged label. 
The Reflexion, TEPA, and TRACE updater prompts remove authentication-only calls and redact credential values from action trajectories. For mechanism diagnostics, we also run a naive full-memory rewrite updater and a corrective-feedback rewrite updater. 

\paragraph{Baseline adaptations.}
We adapt the memory-update and enforcement mechanisms of four methods to the shared AppWorld-P protocol. All four use the same foundation model, actuator, persona, task-stream construction, feedback protocol, and evaluation checkpoints, but update from their own on-policy outcomes. These are mechanism-level adaptations rather than reproductions of the complete original systems. TEPA, TRACE, and Reflexion receive summaries of the current episode's API calls with authentication calls omitted and credential values redacted; ACE receives the current memory and feedback.

\noindent\textbf{ACE.}
Our ACE-style updater~\citep{ICLR2026_ACE} maintains identified preference bullets and asks the model to propose JSON \texttt{add}/\texttt{modify} deltas, which are merged deterministically into the existing memory. After each update, character-trigram vectors identify near-duplicate entries using a cosine-similarity threshold of $0.85$; if more than 12 bullets remain, the oldest are removed. All retained bullets are injected into subsequent episodes. This preserves itemized, incremental updating. 

\medskip
\noindent\textbf{TEPA.}
Our TEPA adaptation~\citep{zhou2026teparevokingstalememories} extracts keyed precedents from the current action summary and feedback, using a supplied attribute vocabulary such as \texttt{sms.ending} and permitting additional keys. Each key has at most one active precedent: a new, nonidentical entry immediately revokes the previous entry, while an exact match to an archived entry can reactivate it. Revoked entries remain archived, and all active entries are injected without task-dependent retrieval. We retain keyed replacement and explicit revocation, but omit evidence-counting, posterior-threshold lifecycle transitions, and the trial-validation stage; the experiment therefore tests this simplified mechanism under fixed user preferences.

\medskip
\noindent\textbf{TRACE.}
Our TRACE adaptation~\citep{zhou2026gettingbetterworkingyou} compiles complaints into checks over API arguments using string equality, prefix, suffix, and containment predicates, including negated prefix, suffix, and containment tests. It retains at most 12 checks and injects their natural-language instructions into context. After an episode, the learned checks inspect the recorded calls; failures produce feedback for a fresh-world retry, with at most three total attempts in both training and evaluation. Unlike the original event-level hooks, this gate operates after a complete episode and uses neither semantic verifiers nor the full rule-resolution lifecycle; it never accesses the ground-truth persona checker, and the final attempt is scored even if the learned checks still fail.

\medskip
\noindent\textbf{Reflexion.}
Our Reflexion adaptation~\citep{NEURIPS2023_Reflexion} generates a one- or two-sentence reflection from the current action summary and complaint after a rejected training episode. Reflections are appended to a FIFO buffer retaining the latest three entries, all of which are injected into subsequent tasks. Previous reflections inform task execution but are not inputs to the reflection-generation prompt, and entries are neither merged nor organized by preference. We thus use bounded verbal memory for cross-task adaptation, without a dedicated same-task reflection-and-retry loop.

\paragraph{Feedback and capacity controls.}
For each violated rule, the updater receives a complaint constructed uniformly from the checker's episode-specific detail. In particular, it may say that a payment note lacks a category tag while withholding which tag is required. Accepted episodes do not trigger an update, avoiding memory drift without new evidence. The corrective diagnostic supplies a fully quantified correction and therefore tests whether a plateau is caused by the information channel rather than the updater alone. 

The four main methods share the feedback-generation protocol and evaluation checkpoints. Their state archives include memory size, update counts, parse failures, and mechanism-specific diagnostics such as ACE order fingerprints, TEPA revocations, TRACE checks fired/failed, and Reflexion evictions. 

Taking the primary ACE as an example, the complaint is inserted into a delta-update prompt rather than a full-memory rewrite:
\begin{promptbox}{ACE Delta-Update Prompt}
You maintain a memory of the user's preferences as itemized bullets.\\
Current memory: <current bullets>\\
Latest feedback from the user: <feedback>\\[2pt]
Propose changes as a JSON array of objects with fields:\\
action = add, text = ..., reason = ...; or\\
action = modify, id = ..., text = ..., reason = ... .\\[2pt]
Each text is one sentence written as a user preference. Output only valid JSON.
\end{promptbox}
For example, a valid delta may add ``Text messages should end with my first name.'' The code-side updater assigns a stable identifier, deduplicates semantically similar bullets, and prunes only when the 12-bullet capacity is exceeded.

\newpage
\newcommand{\Risk}{\mathcal{R}}
\newcommand{\Hctx}{\mathcal{H}^{\mathrm{ctx}}_f}
\newcommand{\Hctrl}{\mathcal{H}^{\mathrm{ctrl}}_f}
\newcommand{\Mbase}{\mathcal{M}_f}
\newcommand{\Cctx}{\mathcal{C}_{\mathrm{ctx}}}
\newcommand{\Cctrl}{\mathcal{C}_{\mathrm{ctrl}}}
\newcommand{\dist}{\operatorname{dist}}
\newcommand{\conv}{\operatorname{conv}}

\section{Proofs}
\subsection{Proof of Theorem~\ref{thm:q1}}\label{app:proof-q1}

\begin{assumption}\label{assump:context_latent}
    For any context harness $c$, (i) the trace distribution is conditionally invariant to context given a concept $p(\tau\mid z,c,\theta)=p(\tau\mid z,\theta)$, and (ii) the concept distribution is fixed by the context and independent of the test task $p(\theta\mid z,c)=p(\theta\mid c)$.
\end{assumption}

\begin{remark}
    Assumption~\ref{assump:context_latent} formalizes the view that a context harness describes user-specific preferences through latent concepts already represented by the frozen model. Once a concept $\theta$ is specified, the context provides no additional effect on the resulting trace distribution. Meanwhile, the context determines the mixture over these latent concepts, while the test task $z$ only governs how a selected concept is expressed in the current interaction.
\end{remark}

\begin{proof}
    We introduce a latent-concept representation of the frozen model. Let $\theta$ denote a behavioral concept represented by $f$, with pretrained prior $p(\theta)$. The frozen policy is represented as
    \begin{equation}\label{eq:pi-f}
        \pi_f(\tau\mid z)=\int p(\tau\mid z,\theta)p(\theta)\,d\theta.
    \end{equation}
    Let $\mathcal C_{\mathrm{ctx}}$ and $\mathcal C_{\mathrm{ctrl}}$ denote the context and control harness classes, respectively. For any $c\in\mathcal C_{\mathrm{ctx}}$, marginalizing over $\theta$ and applying Assumption~\ref{assump:context_latent} gives
    \begin{equation}\label{eq:pi-context}
        \pi_{c\circ f}(\tau\mid z)=\int p(\tau\mid z,c,\theta)p(\theta\mid z,c)\,d\theta=\int p(\tau\mid z,\theta)p(\theta\mid c)\,d\theta.
    \end{equation}
    Define the pretrained policy hull as $\mathcal M_f:=\overline{\operatorname{conv}}\{p(\cdot\mid\cdot,\theta):\theta\in\operatorname{supp}p(\theta)\}$. Since $p(\theta\mid c)$ is a probability distribution, $\pi_{c\circ f}$ is a convex mixture of the pretrained concept-conditioned trace distributions and therefore $\pi_{c\circ f}\in\mathcal M_f$ for every $c\in\mathcal C_{\mathrm{ctx}}$. Hence, defining $\mathcal R_{\mathrm{ctx}}:=\{\pi_{c\circ f}:c\in\mathcal C_{\mathrm{ctx}}\}$, we obtain $\mathcal R_{\mathrm{ctx}}\subseteq\mathcal M_f$.

    For a control harness $c\in\mathcal C_{\mathrm{ctrl}}$, let $K_c(\bar\tau\mid\tau,z)$ denote the intervention kernel that maps a model-proposed trace $\tau$ to an executed trace $\bar\tau$. For any policy $\pi$, define the induced control operator as
    \begin{equation}\label{eq:Tc}
        (\mathcal T_c\pi)(\bar\tau\mid z):=\int K_c(\bar\tau\mid\tau,z)\pi(\tau\mid z)\,d\tau.
    \end{equation}
    Since $K_c(\cdot\mid\tau,z)$ is a probability distribution for every $(\tau,z)$, $\mathcal T_c\pi$ is a valid policy distribution. For each pretrained concept $\theta$, define the control-transformed concept-conditioned trace distribution as
    \begin{equation}\label{eq:q-control}
        q(\bar\tau\mid z,c,\theta):=\int K_c(\bar\tau\mid\tau,z)p(\tau\mid z,\theta)\,d\tau.
    \end{equation}
    Applying the control operator directly to the frozen policy and substituting Equation~\ref{eq:pi-f} gives
    \begin{equation}\label{eq:pi-control}
        (\mathcal T_c\pi_f)(\bar\tau\mid z)=\int K_c(\bar\tau\mid\tau,z)\left[\int p(\tau\mid z,\theta)p(\theta)\,d\theta\right]d\tau=\int q(\bar\tau\mid z,c,\theta)p(\theta)\,d\theta.
    \end{equation}
    Thus, a context harness changes the mixture weights over fixed pretrained concept-conditioned trace distributions, while a control harness can transform the component distributions themselves from $p(\tau\mid z,\theta)$ to $q(\bar\tau\mid z,c,\theta)$ through action-path intervention.

    Define the control-reachable policy set as $\mathcal R_{\mathrm{ctrl}}:=\{\mathcal T_c\pi_f:c\in\mathcal C_{\mathrm{ctrl}}\}$. Unlike $\mathcal R_{\mathrm{ctx}}$, the set $\mathcal R_{\mathrm{ctrl}}$ is not restricted to convex reweightings of the pretrained concept-conditioned trace distributions. In particular, if there exists $c\in\mathcal C_{\mathrm{ctrl}}$ such that $\mathcal T_c\pi_f\notin\mathcal M_f$, then $\mathcal R_{\mathrm{ctrl}}\not\subseteq\mathcal M_f,$ and $\mathcal R_{\mathrm{ctx}}\subseteq\mathcal M_f$.

    Finally, consider a user-optimal policy $\pi_{f_u^*}\in\mathcal R_{\mathrm{ctrl}}\setminus\mathcal R_{\mathrm{ctx}}$ whose risk is separated from the context-reachable class by
    \begin{equation*}
        \Delta_u:=\inf_{\pi\in\mathcal R_{\mathrm{ctx}}}R(\pi)-R(\pi_{f_u^*})>0.
    \end{equation*}
    Then $\epsilon_{\mathrm{app}}(\mathcal C_{\mathrm{ctx}},f_u^*)=\Delta_u>0.$
    Since $\pi_{f_u^*}\in\mathcal R_{\mathrm{ctrl}}$ and $\pi_{f_u^*}$ is user-optimal,
    \begin{equation*}
        \epsilon_{\mathrm{app}}(\mathcal C_{\mathrm{ctrl}},f_u^*)=\inf_{\pi\in\mathcal R_{\mathrm{ctrl}}}R(\pi)-R(\pi_{f_u^*})=0.
    \end{equation*}
    Hence, for user-optimal policies that are unreachable by context-only reweighting but reachable through control intervention, the control class achieves a strictly smaller approximation error.
\end{proof}

\newpage
\subsection{Proof of Theorem~\ref{thm:q2}}\label{app:proof-q2}
Let $U=(U_1,\ldots,U_d)$ denote the user's $d$ independent binary preferences, with $U_j\sim\mathrm{Bernoulli}(1/2)$. For a chosen harness-capacity budget $L$, let $\mathcal C_L$ denote the class of context harnesses containing at most $L$ memory statements. Let $c_{\mathcal C_L}^*\in\arg\min_{c\in\mathcal C_L}R(\pi_{c\circ f})$ denote the population-optimal harness in $\mathcal C_L$, and let $\hat c_{n,L}\in\mathcal C_L$ denote the harness learned from the interaction stream $S_n$. For the theoretical analysis, each fresh task $z$ activates a subset $J(z)\subseteq[d]$ of the preferences, and evaluation is over applicable task--preference pairs $(z,j)$ with $j\in J(z)$, with each preference appearing with probability $1/d$. The loss $\ell_u$ is the indicator that preference $j$ is violated.

\begin{proof}
    We bound the generalization error induced by the finite interaction stream. Each memory statement is associated with a preference that it is intended to represent. Let $m=\lfloor n/L\rfloor$ denote the number of feedback observations allocated to each memory statement. For each statement $\ell\in[L]$, let $X_{\ell,1},\ldots,X_{\ell,m}\in\{0,1\}$ be independent binary feedback observations associated with the preference represented by statement $\ell$, where $X_{\ell,i}=1$ if the $i$-th observation agrees with the corresponding preference, and $\Pr(X_{\ell,i}=1)=1/2+\gamma$ for some $0<\gamma\leq1/2$. Here, $\gamma$ measures the advantage of the feedback over random guessing.
    
    For each statement $\ell\in[L]$, let $E_\ell$ denote the event that the correct preference does not receive a strict majority among its associated feedback observations:
    \begin{align*}
    E_\ell:=\left\{\frac{1}{m}\sum_{i=1}^m X_{\ell,i}\leq\frac{1}{2}\right\}.
    \end{align*}
    Since $\mathbb E[X_{\ell,i}]=1/2+\gamma$, we have
    \begin{align*}
    E_\ell=\left\{\frac{1}{m}\sum_{i=1}^m X_{\ell,i}-\mathbb E[X_{\ell,i}]\leq-\gamma\right\}.
    \end{align*}
    Thus, by Hoeffding's inequality,
    \begin{align*}
    \Pr(E_\ell)\leq\exp(-2m\gamma^2)=\exp\!\left(-2\gamma^2\left\lfloor\frac{n}{L}\right\rfloor\right).
    \end{align*}
    
    Let $E:=\bigcup_{\ell=1}^L E_\ell$. By the union bound,
    \begin{align*}
    \Pr(E)\leq\sum_{\ell=1}^L\Pr(E_\ell)\leq L\exp\!\left(-2\gamma^2\left\lfloor\frac{n}{L}\right\rfloor\right).
    \end{align*}
    On the complement event $E^c$, all target statements are inferred correctly, so $\hat c_{n,L}$ and $c_{\mathcal C_L}^*$ induce identical preference behavior and hence incur the same loss. Therefore, any increase in risk of $\hat c_{n,L}$ relative to $c_{\mathcal C_L}^*$ can occur only on $E$. Since $\ell_u\in\{0,1\}$, taking expectations over $S_n$ gives
    \begin{equation}
    \begin{aligned}
    \mathbb E_{S_n}\!\left[\epsilon_{\mathrm{gen}}(\mathcal C_L,n)\right]
    &=\mathbb E_{S_n}\!\left[R(\pi_{\hat c_{n,L}\circ f})\right]-R(\pi_{c_{\mathcal C_L}^*\circ f})\\
    &\leq\Pr(E)\leq L\exp\!\left(-2\gamma^2\left\lfloor\frac{n}{L}\right\rfloor\right).
    \end{aligned}
    \label{eq:thm2-gen-upper-bound}
    \end{equation}
    For fixed $n$ and $\gamma$, $\lfloor n/L\rfloor$ is nonincreasing in integer $L$, so the exponential factor is nondecreasing. Since the prefactor $L$ is strictly increasing, the generalization bound is increasing in integer $L$.
\end{proof}

\subsection{Proof of Corollary~\ref{cor:q2-excess-risk}}\label{app:proof-q2-cor}
    \begin{proof}
    Consider the information-theoretic framework, where $H(\cdot)$ denotes the information entropy and $I(\cdot;\cdot)$ denotes the mutual information. 
    Since the $d$ preference bits are independent and unbiased, $H(U)=d$. Since every harness in $\mathcal C_L$ contains at most $L$ memory statements of at most $b$ bits each, the population-optimal harness $c_{\mathcal C_L}^*$ satisfies
    \begin{align*}
    I(U;c_{\mathcal C_L}^*)\leq H(c_{\mathcal C_L}^*)\leq Lb.
    \end{align*}
    Therefore,
    \begin{align*}
    H(U\mid c_{\mathcal C_L}^*)=H(U)-I(U;c_{\mathcal C_L}^*)\geq d-Lb.
    \end{align*}
    For each $j\in[d]$, let $p_j^*:=\inf_{\phi_j}\Pr[\phi_j(c_{\mathcal C_L}^*)\neq U_j]$ be the minimum probability of incorrectly recovering $U_j$ from the population-optimal harness, and define $p_L^*:=1/d\cdot\sum_{j=1}^d p_j^*$. By binary Fano's inequality,
    \begin{align*}
    H(U_j\mid c_{\mathcal C_L}^*)\leq h_2(p_j^*),
    \end{align*}
    where $h_2(p)=-p\log_2p-(1-p)\log_2(1-p)$ is the binary entropy function. Using subadditivity of conditional entropy and concavity of $h_2$,
    \begin{align*}
    H(U\mid c_{\mathcal C_L}^*)\leq\sum_{j=1}^dH(U_j\mid c_{\mathcal C_L}^*)\leq\sum_{j=1}^dh_2(p_j^*)\leq d\,h_2(p_L^*).
    \end{align*}
    Combining the lower and upper bounds on $H(U\mid c_{\mathcal C_L}^*)$ gives $h_2(p_L^*)\geq1-Lb/d$. Since $h_2(p_L^*)\geq0$, this yields
    \begin{align*}
    h_2(p_L^*)\geq\left[1-\frac{Lb}{d}\right]_+.
    \end{align*}
    Since an incorrectly represented applicable preference induces a violation, while the user-optimal policy incurs zero violation risk on these binary preferences,
    \begin{align*}
    \mathbb{E}_U\left[\epsilon_{\mathrm{app}}(\mathcal C_L,f_u^*)\right]=\mathbb{E}_U\left[R(\pi_{c_{\mathcal C_L}^*\circ f})-R(\pi_{f_u^*})\right]\geq p_L^*.
    \end{align*}
    Since $p_L^*\in[0,1/2]$ and $h_2$ is increasing on $[0,1/2]$, we obtain
    \begin{equation}
    \mathbb{E}_U[\epsilon_{\mathrm{app}}(\mathcal C_L,f_u^*)]\geq h_2^{-1}\!\left(\left[1-\frac{Lb}{d}\right]_+\right).
    \label{eq:thm2-lower-bound}
    \end{equation}
    
    We next upper bound the approximation error by constructing a harness in $\mathcal C_L$. Let $c_L^{\mathrm{exp}}\in\mathcal C_L$ store the correct values of $\min\{L,d\}$ distinct preferences, one per memory statement. Since each of the $d$ preferences is equally likely to appear in the evaluation, the probability that the evaluated preference is not represented in $c_L^{\mathrm{exp}}$ is
    \begin{align*}
    \frac{d-\min\{L,d\}}{d}=\frac{(d-L)_+}{d}.
    \end{align*}
    Since covered preferences incur no violation and the user-optimal policy has zero violation risk,
    \begin{align*}
    R(\pi_{c_L^{\mathrm{exp}}\circ f})-R(\pi_{f_u^*})\leq\frac{(d-L)_+}{d}.
    \end{align*}
    By the population optimality of $c_{\mathcal C_L}^*$, $R(\pi_{c_{\mathcal C_L}^*\circ f})\leq R(\pi_{c_L^{\mathrm{exp}}\circ f})$. Therefore,
    \begin{equation}
    \epsilon_{\mathrm{app}}(\mathcal C_L,f_u^*)=R(\pi_{c_{\mathcal C_L}^*\circ f})-R(\pi_{f_u^*})\leq\frac{(d-L)_+}{d}.
    \label{eq:thm2-app-upper-bound}
    \end{equation}
    Together, \cref{eq:thm2-lower-bound} and \cref{eq:thm2-app-upper-bound} give
    \begin{equation}
    h_2^{-1}\!\left(\left[1-\frac{Lb}{d}\right]_+\right)\leq \mathbb{E}_U[\epsilon_{\mathrm{app}}(\mathcal C_L,f_u^*)]\leq\frac{(d-L)_+}{d}.
    \label{eq:thm2-app-bound}
    \end{equation}
    
    Combining the approximation upper bound with Theorem~\ref{thm:q2} yields
    \begin{align*}
    \mathbb E_{S_n,U}\!\left[R(\pi_{\hat c_{n,L}\circ f})-R(\pi_{f_u^*})\right]
    &=\mathbb{E}_U\left[\epsilon_{\mathrm{app}}(\mathcal C_L,f_u^*)\right]+\mathbb E_{S_n}\!\left[\epsilon_{\mathrm{gen}}(\mathcal C_L,n)\right]\\
    &\leq\frac{(d-L)_+}{d}+L\exp\!\left(-2\gamma^2\left\lfloor\frac{n}{L}\right\rfloor\right).
    \end{align*}
\end{proof}

\newpage
\subsection{Proof of Theorem~\ref{thm:q3}}
\label{app:proof-q3}
Let $f$ be the frozen foundation model, and fix an update rule $U$ and a feedback protocol throughout the analysis. For a chosen harness-capacity budget $L$, let $\mathcal C_L$ denote the class of context harnesses containing at most $L$ memory statements. Let $\tilde c_{n,L}\in\mathcal C_L$ denote the harness produced by the online update procedure after $n$ interactions, and let $\hat c_{n,L}\in\arg\min_{c\in\mathcal C_L}\widehat R_n(\pi_{c\circ f};S_n)$ be a measurably selected empirical-risk minimizer based on the interaction stream $S_n$.

For the theoretical analysis, let $\mathsf d_L(c,c')$ measure the distance between two harnesses, and assume that $(\mathcal C_L,\mathsf d_L)$ is a complete separable metric space. Let $\mathcal P_1(\mathcal C_L)$ denote the probability distributions over harnesses with finite expected distance from a fixed reference harness $c_0\in\mathcal C_L$. To compare two harness distributions $\nu,\lambda\in\mathcal P_1(\mathcal C_L)$, define their first Wasserstein distance by
\begin{align*}
    W_L(\nu,\lambda):=\inf_{\Gamma\in\Pi(\nu,\lambda)}\int_{\mathcal C_L\times\mathcal C_L}\mathsf d_L(c,c')\,\Gamma(dc,dc'),
\end{align*}
where $\Pi(\nu,\lambda)$ is the set of joint distributions with marginals $\nu$ and $\lambda$. Thus, $\mathsf d_L$ compares individual harnesses, while $W_L$ compares their distributions through the smallest achievable expected harness distance.

\begin{assumption}
    \label{ass:q3-dynamics}
    The online harness sequence under $U$ is a Markov process with a time-independent transition kernel $P_U$. All harness distributions belong to $\mathcal P_1(\mathcal C_L)$, and the induced operator $\mathcal T_U:\mathcal P_1(\mathcal C_L)\to\mathcal P_1(\mathcal C_L)$ satisfies $W_L(\mathcal T_U\nu,\mathcal T_U\lambda)\leq\rho W_L(\nu,\lambda)$ for some $0\leq\rho<1$ and all $\nu,\lambda\in\mathcal P_1(\mathcal C_L)$.
\end{assumption}

\begin{assumption}
    \label{ass:q3-risk}
    The population risk $c\mapsto R(\pi_{c\circ f})$ is bounded and continuous under $\mathsf d_L$. There exist $\mu>0$ and a deterministic sequence $\epsilon_{n,L}\geq0$ such that, for every $n\geq1$, $\mathbb E\mathsf d_L(\tilde c_{n,L},\hat c_{n,L})^2<\infty$,
    \begin{equation*}
        \widehat R_n(\pi_{\tilde c_{n,L}\circ f};S_n)-\widehat R_n(\pi_{\hat c_{n,L}\circ f};S_n)
        \geq\frac{\mu}{2}\mathsf d_L(\tilde c_{n,L},\hat c_{n,L})^2\quad\text{almost surely},
    \end{equation*}
    and
    \begin{equation*}
        \mathbb E\!\left[\left|R(\pi_{\tilde c_{n,L}\circ f})-\widehat R_n(\pi_{\tilde c_{n,L}\circ f};S_n)\right|
        +\left|R(\pi_{\hat c_{n,L}\circ f})-\widehat R_n(\pi_{\hat c_{n,L}\circ f};S_n)\right|\right]\leq2\epsilon_{n,L}.
    \end{equation*}
    For each fixed $L$, the calibration error satisfies $\epsilon_{n,L}\to0$ as $n\to\infty$.
\end{assumption}


\begin{assumption}[Persistent update bias]
    \label{ass:q3-bias}
    For $\beta_{f,n,L}:=W_L(\mathcal T_f\widehat\nu_{n,L},\widehat\nu_{n,L})$, there exists a constant $\beta_0>0$ such that $\liminf_{n\to\infty}\beta_{f,n,L}\geq\beta_0.$
\end{assumption}

\begin{proof}
    For a current harness $c\in\mathcal C_L$, let $P_U(c,\cdot)$ denote the conditional distribution of the next harness over $\mathcal C_L$ after a fresh on-policy interaction and update. Let $\nu_{n,L}$ and $\widehat\nu_{n,L}$ denote the distributions of the online output $\tilde c_{n,L}$ and the empirical-risk minimizer $\hat c_{n,L}$, respectively, over the random interaction stream and algorithmic randomness. Define the distribution update operator by
    \begin{equation}
    \mathcal T_U\nu:=\int_{\mathcal C_L}P_U(c,\cdot)\,\nu(dc).
    \label{eq:online-operator}
    \end{equation}
    By the law of total probability, the online evolution satisfies $\nu_{n+1,L}=\mathcal T_U\nu_{n,L}$. Let
    \begin{equation}
        r_{U,n,L}:=W_L(\nu_{n+1,L},\nu_{n,L}),\quad
        \beta_{U,n,L}:=W_L(\mathcal T_U\widehat\nu_{n,L},\widehat\nu_{n,L}).
        \label{eq:online-residuals}
    \end{equation}
    Here, $r_{U,n,L}$ measures the change in the online output distribution, while $\beta_{U,n,L}$ measures the change induced by applying the update operator to the distribution of the empirical-risk minimizer. The optimization gap is $\epsilon_{\mathrm{opt},n}:=R(\pi_{\tilde c_{n,L}\circ f})-R(\pi_{\hat c_{n,L}\circ f}).$

    Since $(\mathcal C_L,\mathsf d_L)$ is complete and separable, the associated Wasserstein space $(\mathcal P_1(\mathcal C_L),W_L)$ is complete. By Assumption~\ref{ass:q3-dynamics}, $\mathcal T_U$ is a contraction mapping. Banach's fixed-point theorem therefore gives a unique invariant distribution $\nu_\infty$ satisfying
    \begin{align*}
        \mathcal T_U\nu_\infty=\nu_\infty,\quad W_L(\nu_{n,L},\nu_\infty)\leq\rho^nW_L(\nu_{0,L},\nu_\infty).
    \end{align*}
    Thus $\nu_{n,L}\to\nu_\infty$ in $W_L$. By the boundedness and continuity of $c\mapsto R(\pi_{c\circ f})$ in Assumption~\ref{ass:q3-risk},
    \begin{align}
        \mathbb E[R(\pi_{\tilde c_{n,L}\circ f})]
        =\int_{\mathcal C_L}R(\pi_{c\circ f})\,\nu_{n,L}(dc)
        \rightarrow\int_{\mathcal C_L}R(\pi_{c\circ f})\,\nu_\infty(dc).
    \end{align}
    Moreover, since $\nu_{n+1,L}=\mathcal T_U\nu_{n,L}$, the contraction condition gives
    \begin{align*}
        r_{U,n+1,L}
        &=W_L(\mathcal T_U\nu_{n+1,L},\mathcal T_U\nu_{n,L})\\
        &\leq\rho W_L(\nu_{n+1,L},\nu_{n,L})=\rho r_{U,n,L}.
        \end{align*}
    Iterating this inequality yields $r_{U,n,L}\leq\rho^n r_{U,0,L}\to0$.

    Fix $n\geq1$ and write $\nu=\nu_{n,L}$ and $\widehat\nu=\widehat\nu_{n,L}$. By the triangle inequality and Assumption~\ref{ass:q3-dynamics},
    \begin{align*}
        \beta_{U,n,L}
        &=W_L(\mathcal T_U\widehat\nu,\widehat\nu)\\
        &\leq W_L(\mathcal T_U\widehat\nu,\mathcal T_U\nu)+W_L(\mathcal T_U\nu,\nu)+W_L(\nu,\widehat\nu)\\
        &=W_L(\mathcal T_U\widehat\nu,\mathcal T_U\nu)+r_{U,n,L}+W_L(\nu,\widehat\nu)\\
        &\leq(1+\rho)W_L(\nu,\widehat\nu)+r_{U,n,L}.
\end{align*}
    Since $W_L(\nu,\widehat\nu)\geq0$, this yields
    \begin{align*}
        W_L(\nu,\widehat\nu)\geq\frac{[\beta_{U,n,L}-r_{U,n,L}]_+}{1+\rho}.
    \end{align*}
    The joint distribution of $(\tilde c_{n,L},\hat c_{n,L})$ has marginals $\nu$ and $\widehat\nu$, so it belongs to $\Pi(\nu,\widehat\nu)$. By Jensen's inequality and the definition of $W_L$,
    \begin{align*}
        \mathbb E\mathsf d_L(\tilde c_{n,L},\hat c_{n,L})^2
        \geq\left(\mathbb E\mathsf d_L(\tilde c_{n,L},\hat c_{n,L})\right)^2
        \geq W_L(\nu,\widehat\nu)^2.
    \end{align*}
    Adding and subtracting the empirical risks and using the risk calibration condition in Assumption~\ref{ass:q3-risk}, we obtain
    \begin{align*}
        \mathbb E[\epsilon_{\mathrm{opt},n}]
        =&\mathbb E\!\left[\widehat R_n(\pi_{\tilde c_{n,L}\circ f};S_n)-\widehat R_n(\pi_{\hat c_{n,L}\circ f};S_n)\right]
        +\mathbb E\!\left[R(\pi_{\tilde c_{n,L}\circ f})-\widehat R_n(\pi_{\tilde c_{n,L}\circ f};S_n)\right]\\
        &-\mathbb E\!\left[R(\pi_{\hat c_{n,L}\circ f})-\widehat R_n(\pi_{\hat c_{n,L}\circ f};S_n)\right]\\
        \geq&\mathbb E\!\left[\widehat R_n(\pi_{\tilde c_{n,L}\circ f};S_n)-\widehat R_n(\pi_{\hat c_{n,L}\circ f};S_n)\right]-2\epsilon_{n,L}.
    \end{align*}
    Together with the empirical risk growth condition in Assumption~\ref{ass:q3-risk},
    \begin{align*}
        \mathbb E[\epsilon_{\mathrm{opt},n}]
        &\geq\mathbb E\!\left[\widehat R_n(\pi_{\tilde c_{n,L}\circ f};S_n)-\widehat R_n(\pi_{\hat c_{n,L}\circ f};S_n)\right]-2\epsilon_{n,L}\\
        &\geq\frac{\mu}{2}\mathbb E\mathsf d_L(\tilde c_{n,L},\hat c_{n,L})^2-2\epsilon_{n,L}\\
        &\geq\frac{\mu}{2}W_L(\nu,\widehat\nu)^2-2\epsilon_{n,L}.
    \end{align*}
    Combining these inequalities with $r_{U,n,L}\leq\rho^n r_{U,0,L}$ gives
    \begin{align*}
        \mathbb E[\epsilon_{\mathrm{opt},n}]
        &\geq\frac{\mu}{2(1+\rho)^2}[\beta_{U,n,L}-r_{U,n,L}]_+^2-2\epsilon_{n,L}\\
        &\geq\frac{\mu}{2(1+\rho)^2}[\beta_{U,n,L}-\rho^n r_{U,0,L}]_+^2-2\epsilon_{n,L}.
    \end{align*}
    Finally, since $\rho^n r_{U,0,L}\to0$ and $\epsilon_{n,L}\to0$ by Assumption~\ref{ass:q3-risk}, taking the lower limit and applying Assumption~\ref{ass:q3-bias} yields
    \begin{align*}
        \liminf_{n\to\infty}\mathbb E[\epsilon_{\mathrm{opt},n}]
        \geq\frac{\mu\beta_0^2}{2(1+\rho)^2}>0.
    \end{align*}
\end{proof}

\end{document}